\documentclass[journal]{IEEEtran}
\usepackage{amsmath,amsfonts}
\usepackage{amsthm}
\usepackage{algorithmic}
\usepackage{array}
\usepackage{textcomp}
\usepackage{stfloats}
\usepackage{url}
\usepackage{verbatim}
\usepackage{float}
\usepackage{graphicx}
\usepackage{cite}
\usepackage{times}
\usepackage{epsfig}
\usepackage{graphicx}
\usepackage{amsmath}
\usepackage{amssymb}
\usepackage{makecell}
\usepackage{multirow}
\usepackage[ruled]{algorithm2e}
\usepackage{booktabs}
\usepackage{diagbox}
\usepackage{color}
\usepackage[caption=false,font=normalsize,labelfont=sf,textfont=sf]{subfig}
\newtheorem{lemma}{Lemma}
\newtheorem{assumption}{Assumption}
\makeatletter
\newcommand{\removelatexerror}{\let\@latex@error\@gobble}
\makeatother

\begin{document}

\title{Line-based 6-DoF Object Pose Estimation and Tracking With an Event Camera}

\author{Zibin Liu, Banglei Guan, \textit{Member, IEEE}, Yang Shang, Qifeng Yu and Laurent Kneip, \textit{Senior Member, IEEE}
\thanks{The authors would like to acknowledge the funding support provided by projects 12372189 and 62250610225 by the National Natural Science Foundation of China, project 2023JJ20045 by the Hunan Provincial Natural Science Foundation for Excellent Young Scholars, as well as projects 22DZ1201900, 22ZR1441300, and dfycbj-1 by the Natural Science Foundation of Shanghai. \textit{(Corresponding authors: Banglei Guan; Yang Shang)}
	
Zibin Liu, Banglei Guan, Yang Shang, and Qifeng Yu are with the College of Aerospace Science and Engineering, National University of Defense Technology, Changsha 410073, China. (e-mail: liuzibin19@nudt.edu.cn; guanbanglei12@nudt.edu.cn; shangyang1977@nudt.edu.cn; yuqifeng@nudt.edu.cn).
     
Laurent Kneip is with the Mobile Perception Lab of School of Information Science and Technology, ShanghaiTech University, Shanghai 201210, China. (e-mail: lkneip@shanghaitech.edu.cn).}
}

\markboth{ }
{Shell \MakeLowercase{\textit{et al.}}: A Sample Article Using IEEEtran.cls for IEEE Journals}

\IEEEpubid{0000--0000/00\$00.00~\copyright~2024 IEEE}

\maketitle

\begin{abstract}
Pose estimation and tracking of objects is a fundamental application in 3D vision. Event cameras possess remarkable attributes such as high dynamic range, low latency, and resilience against motion blur, which enables them to address challenging high dynamic range scenes or high-speed motion. These features make event cameras an ideal complement over standard cameras for object pose estimation. In this work, we propose a line-based robust pose estimation and tracking method for planar or non-planar objects using an event camera. Firstly, we extract object lines directly from events, then provide an initial pose using a globally-optimal Branch-and-Bound approach, where 2D–3D line correspondences are not known in advance. Subsequently, we utilize event-line matching to establish correspondences between 2D events and 3D models. Furthermore, object poses are refined and continuously tracked by minimizing event-line distances. Events are assigned different weights based on these distances, employing robust estimation algorithms. To evaluate the precision of the proposed methods in object pose estimation and tracking, we have devised and established an event-based moving object dataset. Compared against state-of-the-art methods, the robustness and accuracy of our methods have been validated both on synthetic experiments and the proposed dataset. The source code is available at \url{https://github.com/Zibin6/LOPET}.
\end{abstract}

\begin{IEEEkeywords}
Event camera, pose estimation, object tracking, robust estimation.
\end{IEEEkeywords}

\section{Introduction}
\label{sec:intro}
\IEEEPARstart{O}{bject} pose estimation and tracking is a critical task in computer vision and robotics communities, with a broad range of applications in robotic grasping, augmented reality and virtual reality. In recent years, great progress has been made on this topic, and many different sensor modalities have been proposed, such as monocular cameras~\cite{konishi2016fast}, stereo cameras~\cite{barrois20133d}, and depth cameras~\cite{jain2011real}. Despite these advancements, traditional cameras still encounter challenges in a lot of scenarios. For instance, objects move so fast that the images are blurred; challenging illumination conditions result in rapid changes in lighting during a shot; reflections and specularities can cause confusion in object tracking~\cite{vincent2005monocular}. These issues lead us to wonder if there exist better visual solutions.

\IEEEpubidadjcol
\begin{figure}[t] 
	\centering{\includegraphics[width=9cm,height=9.5cm]{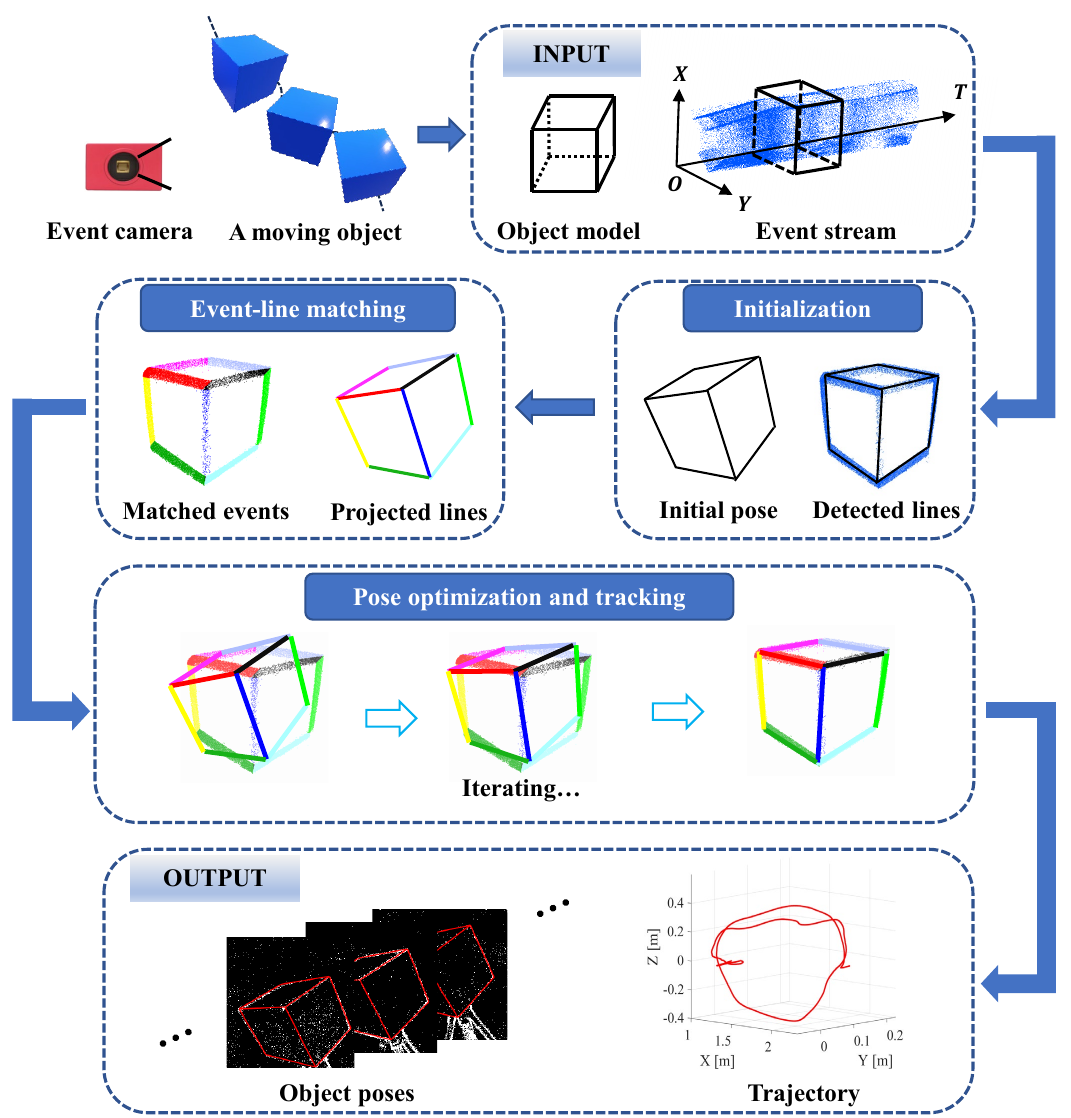}}
	\caption{Block diagram of the proposed method. The core of the method consists of event-based line detection, initial pose estimation, event-line matching, and pose optimization and tracking. Please note that initialization is performed only once. The input to the system comprises the event stream and object model, and the output consists of object poses and trajectories.}
	\label{fig1}
\end{figure}

As a bio-inspired sensor, the event camera has demonstrated great potential to overcome these issues~\cite{gallego_event-based_2022}. Unlike traditional frame-based cameras that capture images at a fixed rate, event cameras asynchronously record per-pixel brightness changes, called ``events''. The output of these events encodes the location, time and polarity. Event cameras possess distinct advantages over traditional cameras: high temporal resolution, low latency, high dynamic range (HDR) and low power consumption. Furthermore, this frame-free approach results in efficiency augmented by circumventing the processing of redundant data (e.g., static background). These characteristics make the sensor particularly well-suited for addressing various visual tasks, especially in challenging scenes, such as high-speed motions and HDR scenes. However, the event camera, being a novel sensor, presents several technical issues that necessitate careful consideration and resolution. To fully explore the utilization of event cameras in object pose estimation and tracking, it is necessary to tackle the following issues.

One of the primary issues with current event cameras is their susceptibility to noise interference. These sensors exhibit greater noise than standard cameras~\cite{gallego_event-based_2022}. To eliminate false positives arising from noise, it is crucial to differentiate events generated by moving objects from noise. Robust estimation is a fundamental tool employed in computer vision to mitigate the impact of noise and outliers~\cite{huber1973robust,rousseeuw1984robust}. To effectively distinguish between events and noise, as well as to further enhance the precision of object pose tracking, it may be beneficial to combine the tracking method with robust estimation. This motivates us to assign different weights to events using robust estimators based on the residuals of event-line distances.	

Another fundamental problem is data association, i.e., establishing correspondences between events and 3D models of objects. The most straightforward approach involves extracting 2D features from events and subsequently aligning them with the 3D features of object models. Each individual event carries little information, lacking texture, color, and other relevant details. Therefore, it is difficult to extract and track event-based feature points accurately~\cite{gallego_event-based_2022}. On the other hand, the majority of events are primarily triggered by the edges of moving objects, which can be effectively represented as multiple interconnected lines. Moreover, lines demonstrate superior stability compared to points~\cite{opromolla2017review}, which exhibit greater resilience to lighting variations and are less susceptible to interference from clutter and noise. Hence, it may be an effective solution for event-based data association by connecting events with the lines of object models.

Regarding the above issues, we propose a line-based object pose estimation and tracking method with an event camera, which is summarized in the block diagram of Fig.~\ref{fig1}. First, the lines of objects are extracted directly from the event cluster. Subsequently, we utilize these lines to estimate the initial pose of objects by Branch-and-Bound (BnB), without any prior knowledge of 2D-3D correspondences. We then employ an event-line matching strategy to establish associations between events and the projected lines of object models. Ultimately, we further refine the object pose and achieve continuous pose tracking by minimizing event-line distances. Additionally, robust estimation algorithms are added to the pose optimization and tracking module, effectively distinguishing events from noise. Due to the lack of appropriate datasets, we provide an event-based object motion dataset, covering several challenging conditions (e.g., sensor noise, textureless specular objects, rapidly changing illumination, etc.). Extensive experiments are evaluated based on synthetic data and our self-collected event dataset, demonstrating the superiority of our method over state-of-the-art methods.

In summary, our main contributions are as follows:
\begin{itemize}
	\item We develop a line-based object pose estimation and tracking framework that works directly on events, including event-based line detection, initial pose estimation, event-line matching, pose optimization and tracking.
	\item We present a line-based object pose initialization method that utilizes a globally-optimal BnB algorithm, without requiring prior knowledge of correspondences between 2D events and 3D models.
	\item We propose a pose optimization method that leverages the unique properties of each event, combined with various robust estimation algorithms. Moreover, we build an event-based object motion dataset for method verification. 
\end{itemize}

The paper is organized as follows: Sec.~\ref{sec:related} reviews related literature on object pose estimation and tracking using either images or events. Sec.~\ref{sec:method} introduces the line-based object pose estimation and tracking method. Sec.~\ref{sec:experi} elaborates experiment designs and results compared against state-of-the-art methods. Finally, Sec.~\ref{sec:concl} presents the conclusion of the paper.

\section{Related Work}
\label{sec:related}
Using vision sensors for object pose estimation and tracking has been extensively studied over the past decades. We briefly review the two primary categories below: those that leverage traditional RGB images and those that employ events.

\subsection{Pose from RGB images}
Feature-based methods work well for monocular cameras, where features are easy to detect and track. Generally, local features (e.g., points, lines, etc.) are extracted from images and subsequently matched with the 3D features of object models. This process establishes 2D-3D feature correspondences for pose estimation of objects. Feature-based methods can be roughly classified into the following two categories.

When 2D-3D correspondences are known, this becomes the Perspective-n-Point (PnP) or Perspective-n-Line (PnL) problems. The minimum requirement for this problem is at least three correspondences, as each correspondence can provide two dimensions of constraints on the 6-DoF pose. If three correspondences of points or lines are present, it is widely recognized as the P3P~\cite{kneip2011novel} or P3L~\cite{dhome1989determination} problems. Nevertheless, the solutions to these two problems are not uniquely determined. If the number of correspondences is greater than or equal to four, a series of efficient and stable algorithms have already been developed, e.g., EPnP~\cite{lepetit2009epnp}, UPnP~\cite{kneip2014upnp}, ASPnL~\cite{xu2016pose}, etc. These algorithms have developed to a mature stage.

When correspondences cannot be established a priori, this gives rise to the simultaneous pose and correspondence determination (SPCD) problem. In such cases, finding correspondences across the two modalities and solving the pose become more complex. Several representative solutions are already available, such as SoftPOSIT~\cite{david2004softposit} and BlindPnP~\cite{moreno2008pose}. However, both solutions require a pose prior. Otherwise, they are susceptible to being trapped in local optima. RANSAC-like approaches~\cite{fischler1981random} do not suffer from this issue, but the computation for a large number of points or lines is time-consuming. Recently, there have emerged several globally optimal methods for the SPCD problem~\cite{Campbell_2017_ICCV,campbell2019alignment}. As one of the most representative works, GOPAC~\cite{Campbell_2017_ICCV} leverages BnB to jointly search the globally optimal camera pose in $SE(3)$. However, the objective function is discrete and presents challenges in optimization. Campbell \emph{et al.}~\cite{campbell2019alignment} cast this problem as a 2D–3D mixture alignment problem and provide the $L_2$ density distance minimization solution. Recent advances in the machine learning community have developed various methods to deal with this problem. For instance, Mask2CAD~\cite{kuo2020mask2cad} detects objects from images, and performs simultaneous 5-DoF alignment and retrieval of 3D CAD models to the detected regions. ROCA~\cite{gumeli2022roca} is an end-to-end approach that achieves accurate and robust 9-DoF alignment and retrieval between 2D images and 3D CAD models. Data-driven methods for object pose estimation and tracking, such as MH6D~\cite{10433529}, HFF6D~\cite{liu2022hff6d}, utilizing RGB-D images as input, effectively address the limitations of monocular depth invisibility~\cite{10043016}.

Feature-based methods rely on sufficient textures on the object for feature extraction. To overcome this limitation, template-based methods can be useful for low-textured or textureless objects~\cite{hinterstoisser2011gradient}. These methods collect a set of template images of objects taken from different viewpoints, and proceed to estimate the object pose by identifying the template that exhibits the highest degree of similarity to the captured image of objects. Muñoz \emph{et al.}~\cite{munoz2016fast} adopt a coarse initialization based on edge correspondences for pose estimation of texture-less objects, and further refine their poses by utilizing RAPID-HOG distances.
Recent success in visual tasks has inspired data-driven methods to apply deep networks for object pose estimation. Nguyen \emph{et al.}~\cite{nguyen2022templates} recognize objects and estimate their 3D poses based on local representations. Their method can generalize to new objects without the need for retraining. Park \emph{et al.}~\cite{Park_2020_CVPR}  propose LatentFusion, an end-to-end differentiable reconstruction and rendering pipeline, specifically designed for pose estimation of unseen objects. The accuracy of template-based methods is significantly influenced by the scale of the template dataset: the larger the dataset, the more closely the estimated parameters align with the actual pose.

\subsection{Pose from events}
Early event-based works focus more on the pose estimation of cameras. At first, researchers deal with simple 3-DoF motion estimation (e.g., planar~\cite{weikersdorfer2013simultaneous} or rotational~\cite{bryner2019event}). More recently, there have been attempts to tackle the problem of event-based simultaneous localization and mapping (SLAM)~\cite{chamorro2022event} and 3D reconstruction~\cite{s11263-017-1050-6}. Chamorro \emph{et al.}~\cite{chamorro2022event} propose an event-based line SLAM method, which estimates the 6-DoF pose of an event camera, and achieves scene reconstruction simultaneously. Gentil \emph{et al.}~\cite{le_gentil_idol_2020-1} present a line-based framework for IMU-DVS odometry, which directly leverages asynchronous events without any frame-like accumulation.

Owing to its specific characteristics, the event camera has the ability to track fast-moving objects. The first 6-DoF object pose estimation algorithm using an event camera is proposed by~\cite{reverter2016neuromorphic}. Given prior knowledge of the object model and its initial pose, their method estimates and tracks moving objects by relating events to the 3D points of objects. Whereafter, an event-based PnP algorithm is presented by~\cite{reverter2016event} to estimate object poses and enable continuous tracking. This algorithm is mathematically formulated as a least squares problem. In space applications, Jawaid \emph{et al.}~\cite{jawaid2022towards} propose an event-based satellite pose estimation method, using classic learning-based approaches to detect the satellite and predict the 2D positions of its landmarks. Subsequently, the corresponding 2D-3D points are fed into the PnP solver to determine the satellite poses. Zheng \emph{et al.}~\cite{zheng2022spike} propose a novel bio-inspired unsupervised learning framework to detect and track moving objects. The method utilizes an alternative type of neuromorphic vision sensors, namely spiking cameras~\cite{huang20221000}, to achieve favorable tracking outcomes.

Furthermore, event cameras have also been utilized for human motion estimation. Calabrese \emph{et al.}~\cite{calabrese2019dhp19} propose the first event-based 3D human motion estimation algorithm using multiple event cameras. Recently, Xu \emph{et al.}~\cite{xu2020eventcap} achieve the 3D capture of high-speed human motions with a single event camera. Their method captures human motions at millisecond resolution by combining model-based optimization and CNN-based human pose detection. Nonetheless, their method encounters considerable obstacles in addressing large-scale occlusions and topological changes.

Data-driven methods have experienced rapid development in the field of object pose estimation and tracking, particularly with the advancement of computational power in recent years.
However, these methods necessitate manual labeling of training data, which is time-consuming and does not promote diversity across various scenes and object categories~\cite{10433529}. Also, achieving real-time performance proves challenging for the majority of these works, even with the utilization of high-performance GPUs~\cite{8565885}.

In this paper, we seek to continuously recover 6-DoF that defines the object position and orientation relative to the camera or, equivalently, the pose of the camera relative to the object. Different from the methods mentioned above, we do not require a pose prior. Moreover, we fully exploit the characteristics of each event and combine them with robust estimation algorithms to efficiently reduce the impact of various interferences, such as outliers, sensor noise, and background clutter.

\section{Line-based Object Pose Estimation and Tracking}
\label{sec:method}
This section describes the details of our proposed method. First of all, the lines of an object are detected directly from the event cluster (Sec.~\ref{method0}), and employed for the initial pose estimation without known 2D-3D line correspondences (Sec.~\ref{method1}). After initialization, the event-line matching is adopted to continuously establish the correspondences between the events and object lines (Sec.~\ref{method2}). Finally, the object poses are refined and tracked using non-linear optimization combined with M-estimation, S-estimation, and MM-estimation (Sec.~\ref{method4}).
\begin{figure}[t] 
	\centering{\includegraphics[width=\linewidth]{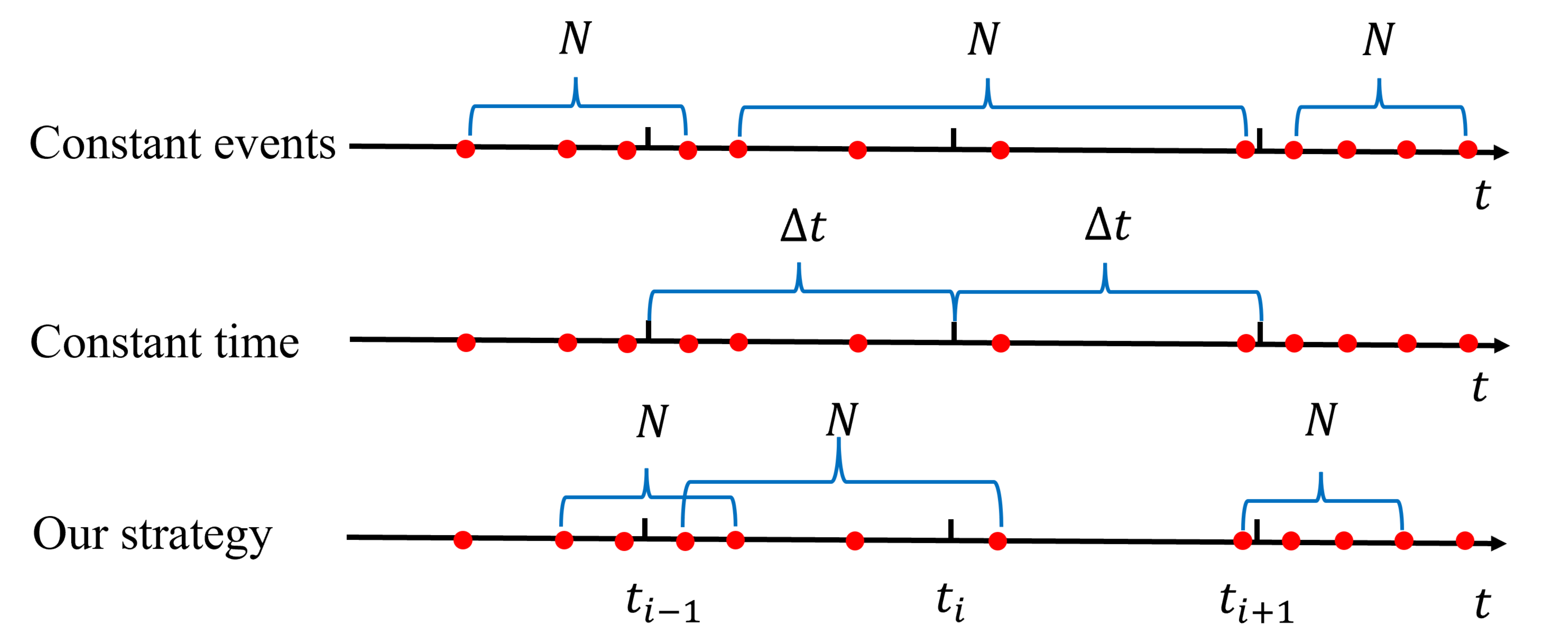}}
	\caption{Hybrid event clustering strategy. We use a spatio-temporal window of events, containing $N$ events closest to the time $t_i$. The parameters $N$ and $ \Delta  t $ are adaptively adjusted based on the number of events. Red dots correspond to events, and the bounds of the event cluster are marked in blue.}
	\label{fig2}
\end{figure}

\subsection{Event-based Line Detection}
\label{method0}
The event camera outputs asynchronous event streams, which fundamentally differ from conventional cameras. Each event ${{\mathbf{e}}_{i}}$ is defined by its image position, timestamp, and binary polarity. Events in a spatio-temporal window are usually processed together. Choosing the appropriate window size is critical for object pose estimation and tracking. Our approach aims to strike a balance between constant events and time, which enables us to preserve the temporal nature of events and accurately describe the pose state of an object at time $ t_i $. Inspired by~\cite{vidal2018ultimate} and~\cite{rebecq_real-time_2017}, we adopt a hybrid event clustering strategy that combines the two existing strategies: constant events $ N $ and constant time $ \Delta  t $. As depicted in Fig.~\ref{fig2}, a new spatio-temporal window of events is created by accumulating $ N $ events closest to the time $ t_i $. The rapid or slow movement of objects, as well as lighting changes, may lead to sharp fluctuations in the number of events. To tackle this issue, we devise a parameter-adaptive hybrid event clustering strategy to dynamically adjust $ N $ and $ \Delta  t $. 

If there are too few events within the spatio-temporal window from $ {t}_{i-1} $ to $ {t}_{i+1} $, the events within the next window $ {t}_{i+2} $ are also included, and then update ${t_i} = {{\left( {{t_{i - 1}} + {t_{i + 2}}} \right)} \mathord{\left/{\vphantom {{\left( {{t_{i - 1}} + {t_{i + 2}}} \right)} 2}} \right.\kern-\nulldelimiterspace} 2}$. Conversely, if there are too many events within this window, only the $ N $ events closest to $ t_i $ are retained, as they more precisely reflect the pose state of the object at time $ t_i $.

Events are triggered by object contours when there is relative motion between the object and the sensor. These contours can be modeled and parametrized as a group of lines. Lines have been extensively utilized in event-based visual perception, due to their higher detectability in event streams compared to other features, such as corners. Moreover, lines are more stable than points and less likely to be disturbed by noise and clutter~\cite{opromolla2017review}. Therefore, we leverage the object lines detected from events for initial pose estimation. A common method for event-based line detection is to accumulate events into 2D images and apply traditional line detection methods. However, this method compresses events directly into 2D images, without fully utilizing temporal information of events.

Here, we adopt an event-based line detection method by converting events to 3D point clouds in the space-time volume. On short time scales, lines that leave traces of events in the spatio-temporal dimension are nearly planar~\cite{peng_xin_continuous_2021}. Based on this assumption, we start with event clustering over a short time interval to obtain point clouds. Point cloud segmentation methods are mature and widely applied~\cite{lu_pairwise_2016}, inspiring us to divide the event cluster into multiple 3D planes. Subsequently, we perform plane fitting and extract the edge lines, which correspond to object lines at the beginning and end of the time interval.
\begin{figure}[t] 
	\centering{\includegraphics[width=8.5cm]{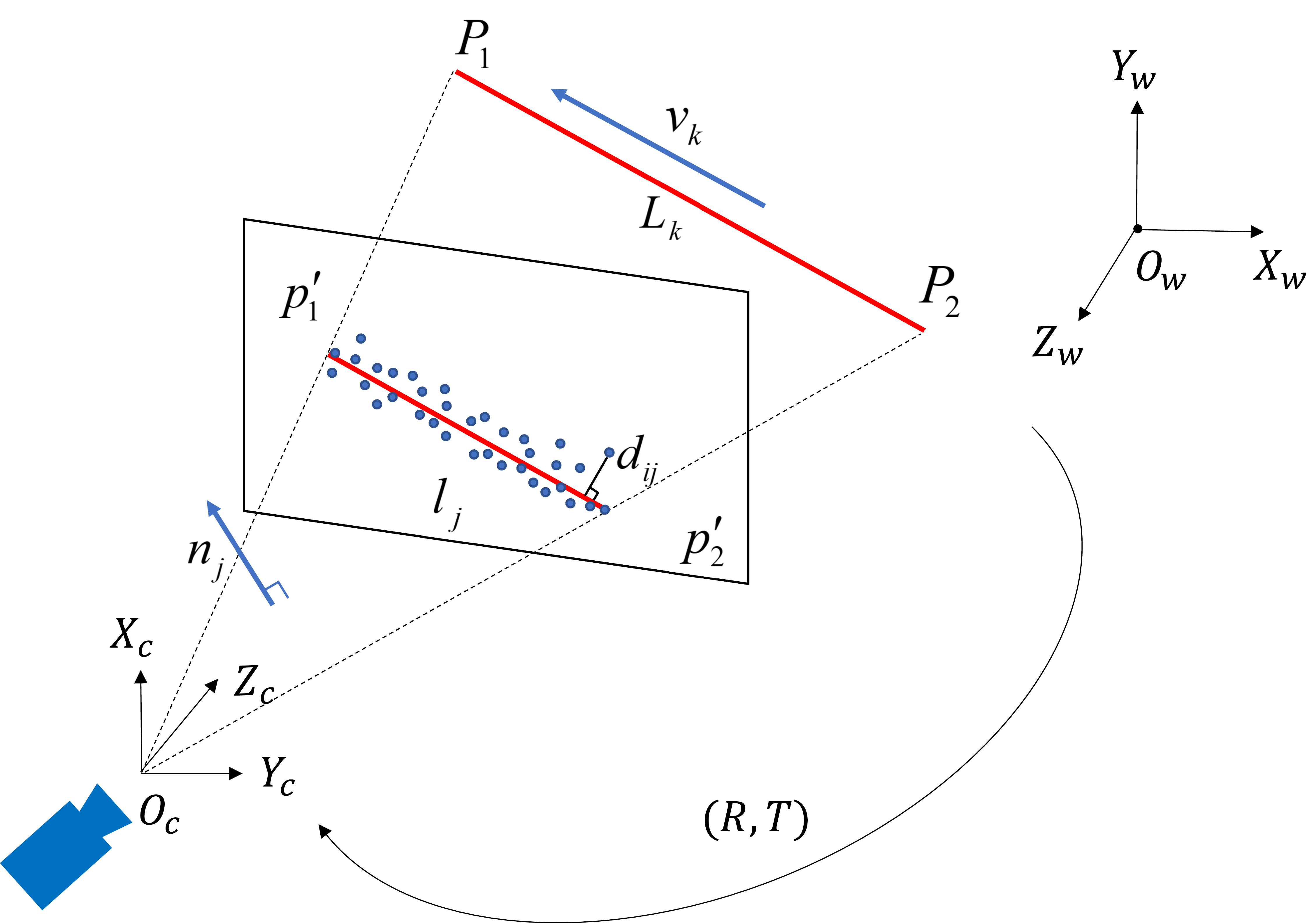}}
	\caption{The relative pose between the events and the projected lines. The 3D line is parameterized by its two endpoints $\mathbf P_{1}$ and $\mathbf P_{2}$. The projected line lies at the intersection of the plane $\mathbf P_{1} \mathbf{O_c} \mathbf P_{2}$ and the event plane.}
\label{fig3}
\end{figure}

\subsection{Initial Pose Estimation}
\label{method1}
Pose optimization is the non-convex optimization problem. Many methods initially calculate rotation matrices, thereby transforming the problem into a linear least-squares one, which can be effectively and reliably solved~\cite{7139836}. Liu \emph{et al.}~\cite{liu2020globally} employ the BnB algorithm to address the PnL problem. Inspired by this, we further endeavor to address the more complex pose estimation problem when 2D-3D line correspondences are unknown. To be specific, we utilize BnB to conduct a global search for the optimal rotation $\mathbf{R}\in SO(3)$, followed by employing a linear solution to calculate the translation vector $\mathbf{T}\in {{\mathbb{R}}^{3}}$.

For each 3D-2D line correspondence, such as $ \mathbf L_{k}$ and $ \mathbf l_{j}$, there exists a plane passing through both lines and the origin $ O_c $ of the camera coordinate system. The unit normal vector of this plane is represented by $\mathbf{n}_{j}$. Let ${\mathbf{v}}_{k}$ denote the unit direction vector of line $ \mathbf L_{k}$, and let ${\mathbf{P}}_{k}$ be an arbitrary point on this line, as depicted in Fig.~\ref{fig3}. Geometrically, there is a significant constraint~\cite{xu2016pose}:
\begin{equation}
\mathbf{n}_{j}^\top \mathbf{Rv}_{k}=0.
\label{eq1-1}
\end{equation}

In correspondence-free cases, we aim to find the optimal rotation $\mathbf{R}$ by maximizing the cardinality of the inlier set $E\left( \mathbf{R} \right)$, which can be formulated as: 
\begin{equation}
E\left( \mathbf{R} \right)=\sum\limits_{j}{\underset{k}{\mathop{\max }}\,\left\lfloor \left| \angle \left( {{\mathbf{n}}_{j}},\mathbf{R}{{\mathbf{v}}_{k}} \right)-\frac{\pi }{2} \right|\le \epsilon_{\min }  \right\rfloor },
\label{eq1-2}
\end{equation}
where $\epsilon_{\min }$ is the inlier threshold and $\left\lfloor \cdot  \right\rfloor $ is an indicator function that returns 1 if the inner condition is true and 0 otherwise. 

To solve this non-convex cardinality maximisation problem, we utilize the BnB algorithm to find a globally-optimal solution $\mathbf{R}$. Here, rotations are denoted as $\mathbf{r}\in {{\mathbb{R}}^{3}}$ in the form of axis-angle parameters, enabling the search space of rotations within a solid ball of radius $\pi$. We further employ a 3D cube ${{\left[ -\pi ,\pi  \right]}^{3}}$ that circumscribes the $\pi$-ball as the rotation domain. The rotation matrix represented by $\mathbf{r}$ can be obtained by Rodrigues’ formula.

The BnB algorithm subdivides the solution domain into smaller branches, and then calculates upper and lower bounds of the optimum within each of these branches, which gradually converge as the branch size tends to zero~\cite{Campbell_2017_ICCV}.

For each cube-shaped rotation branch, whose center is $\mathbf{r_0}$, we first derive the lower bound. Any value of the function within the rotation branch will naturally not surpass the maximum value~\cite{Campbell_2017_ICCV}, i.e.,
\begin{equation}
\underset{\mathbf{r}}{\mathop{\max }}\,E\left( {{\mathbf{R}}_{\mathbf{r}}} \right)\ge E\left( {{\mathbf{R}}_{{{\mathbf{r}}_{\mathbf{0}}}}} \right),
\label{eq1-4.2}
\end{equation}
where $\mathbf{R}_\mathbf{r_0}$ is the matrix form of rotation $\mathbf{r_0}$. Then, the lower bound $E_L$ can be determined as
\begin{equation}
{{E}_{L}}=E\left( {{\mathbf{R}}_{{{\mathbf{r}}_{\mathbf{0}}}}} \right).
\label{eq1-4}
\end{equation}

Further, we employ the following two Lemmas~\cite{hartley2009global} to derive the upper bound.
\begin{lemma}
\label{lemma1}
Let ${\mathbf{R}_\mathbf{a_1}}$, ${\mathbf{R}}_{\mathbf{a_2}}$ be rotation matrices, $\mathbf{a_1}$, $\mathbf{a_2}$ their corresponding axis-angle representations. For a vector $\mathbf{v}\in {{\mathbb{R}}^{3}}$:
\begin{equation}
\angle \left( {\mathbf{R}_\mathbf{a_1}} \mathbf{v},{{\mathbf{R}}_{\mathbf{a_2}}}\mathbf{v} \right)\le \left\| \mathbf{a_1}-\mathbf{a_2} \right\|
\label{eq1-5},
\end{equation}
where $\angle \left( \cdot ,\cdot  \right)$ denotes the angle between the two rotated vectors, which is less than or equal to the Euclidean distance between their rotations' angle-axis representations.

The proof of Lemma 1 can be found in~\cite{brown2019family}.
\end{lemma}

\begin{lemma}
\label{lemma2}
Given a rotation $\mathbf{R}_\mathbf{r}$ within the cube-shaped rotation branch of side-length $\delta_{r}$, for any point $\mathbf{p}$:
\begin{equation}
\begin{aligned}
	\angle\left(\mathbf{R}_{\mathbf{r}} \mathbf{p}, \mathbf{R}_{\mathbf{r}_{0}} \mathbf{p}\right) 
	\leqslant \min \left( {\sqrt{3}{{\delta }_{r}}}/{2}\;,\pi  \right).
\end{aligned}
\label{eq1-6}
\end{equation}
\end{lemma}

\begin{proof}
According to Lemma~\ref{lemma1}, inequality can be derived as:
\begin{equation}
\begin{aligned}
	\angle\left(\mathbf{R}_{\mathbf{r}} \mathbf{p}, \mathbf{R}_{\mathbf{r}_{0}} \mathbf{p}\right) & \leqslant \min \left(\left\|\mathbf{r}-\mathbf{r}_{0}\right\|, \pi\right)\\&
	\leqslant \min \left( {\sqrt{3}{{\delta }_{r}}}/{2}\;,\pi  \right),
\end{aligned}
\label{eq1-7}
\end{equation}
where Eq. (\ref{eq1-7}) follows that the maximum value of $\left\| \mathbf{r}-{{\mathbf{r}}_{0}} \right\|$ amounts to ${\sqrt{3}{{\delta }_{r}}}/{2}\;$, corresponding to the space diagonal of the rotation cube.
\end{proof}

Let $\mu \triangleq \min \left\{ {\sqrt{3} {\delta }_{r} }/{2,\pi}\;\right\}$. According to triangle inequality in spherical geometry~\cite{liu2020globally},
\begin{equation}
\angle \left( {{\mathbf{n}}_{j}},{{\mathbf{R}}_{\mathbf{r}}}{{\mathbf{v}}_{k}} \right)-\text{ }\angle \left( {{\mathbf{n}}_{j}},{{\mathbf{R}}_{{{\mathbf{r}}_{\mathbf{0}}}}}{{\mathbf{v}}_{k}} \right)\le \angle \left( {{\mathbf{R}}_{\mathbf{r}}}{{\mathbf{v}}_{k}},{{\mathbf{R}}_{{{\mathbf{r}}_{\mathbf{0}}}}}{{\mathbf{v}}_{k}} \right)\le \mu ,
\label{eq1-9}
\end{equation}
\begin{equation}
\angle \left( {{\mathbf{n}}_{j}},{{\mathbf{R}}_{\mathbf{r}}}{{\mathbf{v}}_{k}} \right)-\angle \left( {{\mathbf{n}}_{j}},{{\mathbf{R}}_{{{\mathbf{r}}_{\mathbf{0}}}}}{{\mathbf{v}}_{k}} \right)\ge -\angle \left( {{\mathbf{R}}_{\mathbf{r}}}{{\mathbf{v}}_{k}},{{\mathbf{R}}_{{{\mathbf{r}}_{\mathbf{0}}}}}{{\mathbf{v}}_{k}} \right)\ge -\mu ,
\label{eq1-10}
\end{equation}
therefore,
\begin{equation}
\left| \angle \left( {{\mathbf{n}}_{j}},{{\mathbf{R}}_{{{\mathbf{r}}_{\mathbf{0}}}}}{{\mathbf{v}}_{k}} \right)-\angle \left( {{\mathbf{n}}_{j}},{{\mathbf{R}}_{\mathbf{r}}}{{\mathbf{v}}_{k}} \right) \right|\le \mu. 
\label{eq1-11}
\end{equation}

Then, we can derive the following formula
\begin{align}
& \left| \angle \left( {{\mathbf{n}}_{j}},\mathbf{R}_\mathbf{r}{{\mathbf{v}}_{k}} \right)-\frac{\pi }{2} \right| \\ 
& =\left| \angle \left( {{\mathbf{n}}_{j}},{{\mathbf{R}}_{{{\mathbf{r}}_{\mathbf{0}}}}}{{\mathbf{v}}_{k}} \right)-\frac{\pi }{2}-\left( \angle \left( {{\mathbf{n}}_{j}},{{\mathbf{R}}_{{{\mathbf{r}}_{\mathbf{0}}}}}{{\mathbf{v}}_{k}} \right)-\angle \left( {{\mathbf{n}}_{j}},\mathbf{R}_\mathbf{r}{{\mathbf{v}}_{k}} \right) \right) \right| \\ 
& \ge \left| \angle \left( {{\mathbf{n}}_{j}},{{\mathbf{R}}_{{{\mathbf{r}}_{\mathbf{0}}}}}{{\mathbf{v}}_{k}} \right)-\frac{\pi }{2} \right|-\left| \angle \left( {{\mathbf{n}}_{j}},{{\mathbf{R}}_{{{\mathbf{r}}_{\mathbf{0}}}}}{{\mathbf{v}}_{k}} \right)-\angle \left( {{\mathbf{n}}_{j}},\mathbf{R}_\mathbf{r}{{\mathbf{v}}_{k}} \right) \right| \\ 
& \ge \left| \angle \left( {{\mathbf{n}}_{j}},{{\mathbf{R}}_{{{\mathbf{r}}_{\mathbf{0}}}}}{{\mathbf{v}}_{k}} \right)-\frac{\pi }{2} \right|-\mu.
\end{align}

It follows that 
\begin{align}
& \left \lfloor \left| \angle \left( {{\mathbf{n}}_{j}},{{\mathbf{R}}_{\mathbf{r}}}{{\mathbf{v}}_{k}} \right)-\frac{\pi }{2} \right|\le {{\epsilon }_{\min }} \right \rfloor\\ 
& \le \left\lfloor \left| \angle \left( {{\mathbf{n}}_{j}},{{\mathbf{R}}_{{{\mathbf{r}}_{\mathbf{0}}}}}{{\mathbf{v}}_{k}} \right)-\frac{\pi }{2} \right|-\mu \le {{\epsilon }_{\min }} \right\rfloor \\ 
& =\left\lfloor \left| \angle \left( {{\mathbf{n}}_{j}},{{\mathbf{R}}_{{{\mathbf{r}}_{\mathbf{0}}}}}{{\mathbf{v}}_{k}} \right)-\frac{\pi }{2} \right|\le \left( \mu +{{\epsilon }_{\min }} \right) \right\rfloor.
\end{align}

Therefore, the upper bound $E_U$ of the objective function Eq. (\ref{eq1-2}) can be derived as:
\begin{equation}
{{E}_{U}}=\sum\limits_{j}{\underset{k}{\mathop{\max }}\,\left\lfloor \left| \angle \left( {{\mathbf{n}}_{j}},{{\mathbf{R}}_{{{\mathbf{r}}_{\mathbf{0}}}}}{{\mathbf{v}}_{k}} \right)-\frac{\pi }{2} \right|\le \left( \epsilon_{\min } +\mu  \right) \right\rfloor }.
\label{eq1-13}
\end{equation}

By comparing the upper and lower bounds, we can eliminate old branches and refine new branches until convergence. The convergence criterion of the BnB algorithm is met when the upper bound equals the lower bound. The optimal rotation has been obtained, and simultaneously, the correspondences of the 2D-3D lines have also been determined.

The next step is to compute the translation vector. Theoretically, the translation vector can also be obtained using BnB. However, solving for rotation and translation sequentially requires two independent BnB processes, making it challenging to guarantee that the optimal rotation and translation are globally optimal solutions~\cite{liu2020globally}. Therefore, we determine the translation vector by applying another geometrical constraint~\cite{xu2016pose}:
\begin{equation}
\mathbf{n}_{j}^\top\left( \mathbf{R}{{\mathbf{P}}_{k}}+\mathbf{T} \right)=0.
\label{eq1-14}
\end{equation}

This constraint denotes that following the transformation of ${\mathbf{P}}_{k}$ into the camera coordinate system, it should be perpendicular to $\mathbf{n}_{j}$. After determining the rotation and correspondences, the translation vector $\mathbf{T}$ can be linearly computed.

\subsection{Event-line Matching}
\label{method2}
The data association between events and lines is a critical part of our method, as it enables us to establish correspondences between 2D events and 3D models, while also effectively eliminating outliers. Each event is assigned to each candidate line based on their distances. Events that are too far from the candidate line may be generated by background or sensor noise, which should be removed. The choice of these distance thresholds is a tradeoff between latency and accuracy. Using more events may lead to more accurate pose estimation results, while higher latency would be introduced. Hence, these thresholds need to be dynamically adjusted based on the motion state of the object.

We have initially taken into account the rotational motion of the object. The relative rotation ${\Delta \mathbf{R}}$ from time ${{t}_{i}}$ to ${t}_{i+1}$ can be parameterized as a continuous time function by a local angular velocity $w$:
\begin{equation}
\begin{aligned}
{\Delta \mathbf R}&=\exp \left(\widehat{w}\left(t_{i+1}-t_{i}\right)\right)\\
&=\cos (\theta) \mathbf{I}+(1-\cos (\theta)) \mathbf{b} \mathbf{b}^{\top}+\sin (\theta) \widehat{\mathbf{b}},
\end{aligned}
\label{eqd-1}
\end{equation}
where $\left( \mathbf{b},\theta  \right)$ are the axis-angle parameters of the relative rotation ${\Delta \mathbf{R}}$, and $\widehat{\cdot }$ is the cross-product matrix. 

Further, we take account of the rotational and translational motion of the object at the same time. On short time scales (e.g., in the millisecond range), it is reasonable to assume that the motion of an object is smooth with constant angular velocity $w $ and linear velocity $\nu $~\cite{peng_xin_continuous_2021}. The relative translation ${\Delta \mathbf{{T}}}$ from time ${{t}_{i}}$ to ${t}_{i+1}$ is given by:
\begin{equation}
\begin{aligned}
{\Delta \mathbf T} &=\mathbf{J} \nu \left(t_{i+1}-t_{i}\right), \\
\text { where } \mathbf{J} &=\frac{\sin (\theta)}{\theta} \mathbf{I}+\left(1-\frac{\sin (\theta)}{\theta}\right) \mathbf{b b}^{\top}+\frac{1-\cos \theta}{\theta} \widehat{\mathbf{b}}.
\end{aligned}
\label{eqd-2}
\end{equation}

Events are triggered sequentially as the camera or object moves. Within a short period of time, event streams generated by moving objects are smooth and continuous. Once the object pose of an event cluster is determined, the initial pose of the adjacent event cluster at the next time can be inferred using the constant-velocity motion model. 
Given the pose ${{\mathbf{R}}_{i}}$, ${{\mathbf{T}}_{i}}$ of the object at time ${{t}_{i}}$, along with constant angular velocity $w$ and linear velocity $\nu $, the initial pose of the next event cluster at time ${{t}_{i+1}}$ can be obtained:
\begin{equation}
\mathbf{R}_{i+1}^{0}={{\mathbf{R}}_{i}}\Delta {{\mathbf{R}}^{-1}}, 
\label{eqd-3}
\end{equation}
\begin{equation}
\mathbf{T}_{i+1}^{0}={{\mathbf{T}}_{i}}-{{\mathbf{R}}_{i}}\Delta {{\mathbf{R}}^{-1}}\Delta \mathbf{T}. 
\label{eqd-4}
\end{equation}

\begin{figure}[t] 
\centering{\includegraphics[width=8cm]{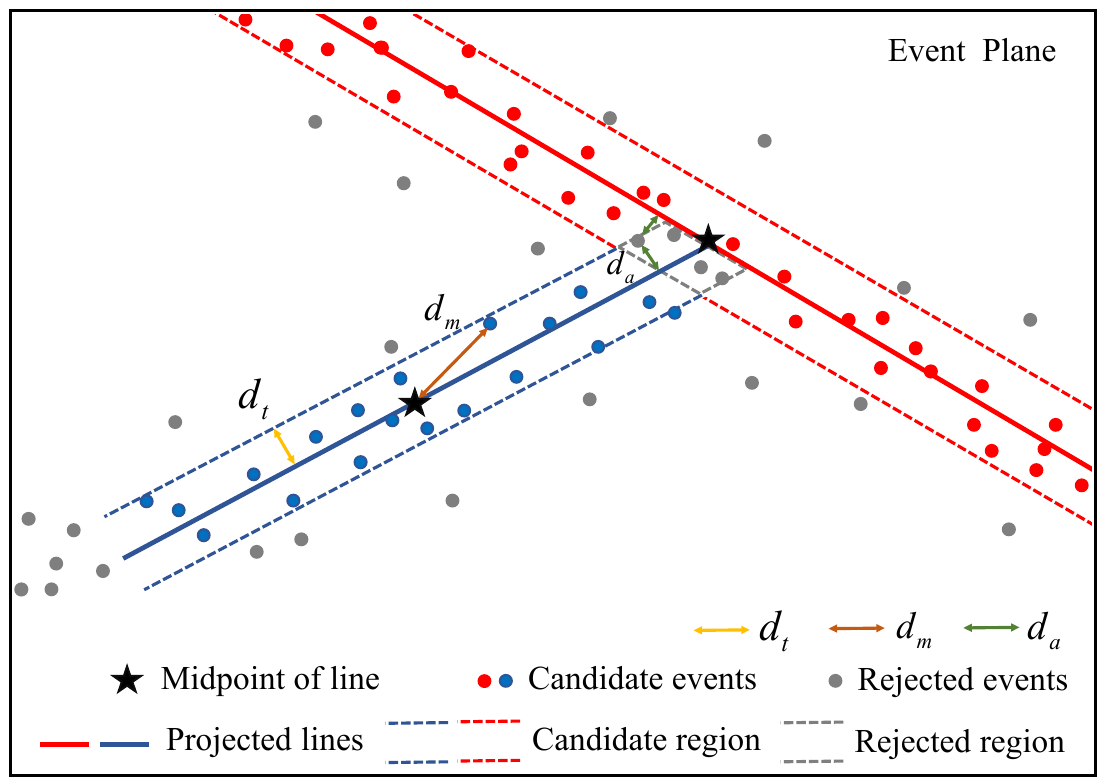}}
\caption{Illustration of the event-line matching strategy. Events are selected or rejected based on the threshold $d_a$, $d_t$ and $ d_m$.}
\label{fig4}
\end{figure}

Combining the constant-velocity motion model, we proceed to establish the association between events and lines. Theoretically, events are triggered by the nearest line. However, noise may also be present in practical situations. We propose an event-line matching strategy that fully utilizes the geometric distance between events and lines, thereby reducing the interference of noise. Firstly, a candidate region is established in near the line. An event becomes a candidate event when its distance to the nearest line is less than $d_t$, and its distance to the midpoint of that line is less than $d_m$. This establishes a rectangular candidate region, ensuring that candidate events fall within the scope of the line. If an event is in close proximity to two lines, it is not possible to directly ascertain its correspondence to either line. To address this situation, we construct a rejected region. If an event is in close proximity to the two nearest lines, each with a distance less than $d_a$, the event is rejected from being associated with any line. The event-line matching strategy is illustrated in Fig.~\ref{fig4}. The distance thresholds $d_t$, $d_m$, $d_a$ can be empirically determined.

The object lines that move outside the field of view (FoV) of the camera or are situated on the backside of the model are considered invisible and, consequently, ought to be discarded. The remaining events will only be assigned to visible lines. Given the 3D model and pose of objects, the visibility issue of the lines can be transformed into assessing the visibility of their endpoints.

Firstly, assess whether the endpoints lie within the FoV of the camera, by projecting the endpoints onto the event plane using the object pose. Subsequently, establish a ray by connecting the optical center with the endpoints. If the ray intersects with multiple planes of the object, visibility is determined by comparing the distances between these intersection points and the optical center of the camera. Conversely, if there is only a single intersection point, specifically referring to the endpoint itself, then the endpoint is considered visible. If both endpoints are visible, the line is considered to be visible. Otherwise, the line is deemed to be invisible. In the case where only one endpoint is visible, uniform sampling of the line is conducted. The aforementioned steps are then repeated to determine the visibility of the line segments. Furthermore, a more lightweight alternative solution is available by ignoring those lines that correspond to only a few events through event-line matching. Invisible lines will be removed to improve the search speed and accuracy of the event-line matching.

\subsection{Event-based Pose Optimization and Tracking}
\label{method4}
With the initial pose, the 3D lines of objects can be projected onto the event plane to obtain 2D lines. Let $\mathbf P_{1}$ and $ \mathbf P_{2} \in \mathbb{R}^3$ be the endpoints of the line $ \mathbf L_{k}$, $\mathbf p_{1}$ and $ \mathbf p_{2} \in \mathbb{R}^2$ be the endpoints of the 2D projected line $\mathbf l_{j}$ on the event plane, and $\mathbf {p}_{\mathrm{1}}^{\mathrm{h}}$ and $\mathbf{p}_{\mathrm{2}}^{\mathrm{h}}$ are their homogeneous coordinates, as shown in Fig.~\ref{fig3}. The line coefficients of $\mathbf l_{j}$ are:
\begin{equation}
\mathbf l_{j}=\frac{\mathbf{p}_{\mathrm{1}}^{\mathrm{h}} \times \mathbf{p}_{\mathrm{2}}^{\mathrm{h}}}{\left|\mathbf{p}_{\mathrm{1}}^{\mathrm{h}} \times \mathbf{p}_{\mathrm{2}}^{\mathrm{h}}\right|}=[l_{jx},l_{jy},l_{jz}]^\top.
\label{eq2}
\end{equation}

After performing event-line matching, an event $\mathbf e_{i} $ within the event cluster corresponds to a unique projected line $\mathbf l_{j}$, and the event-line distance can be obtained:
\begin{equation}
{{d}_{ij}}=\frac{\mathbf e_{i}^{T}{\mathbf{l}_{j}}}{\sqrt{{{l_{jx}}^{2}}+{{l_{jy}}^{2}}}}.
\label{eq3}
\end{equation}

The object pose can be continuously tracked by minimizing event-line distances, which is defined as the sum of event-line distances ${d}_{ij}$ between the events and the projected lines. Note that in practice, just like standard vision sensors, event cameras also have fixed pattern noise (FPN)~\cite{gallego_event-based_2022}. Hence, event streams contain a lot of noise and outliers. To attenuate the leverage of influential noise and outliers, a robust pose tracking method is proposed by assigning different weights ${\omega }_{ij}$ to each event. The weights are determined based on event-line distances, i.e., events with smaller residuals are given larger weights, while events with larger residuals are assigned smaller weights. The assigned weights are based on the Our-MM method as illustrated in Fig.~\ref{fig5}, which will be discussed further in the subsequent sections.
\begin{figure}[t]
\centering
\subfloat[Distance]{
\begin{minipage}[t]{0.49\linewidth}
	\includegraphics[width=4.1cm,height=3.7cm]{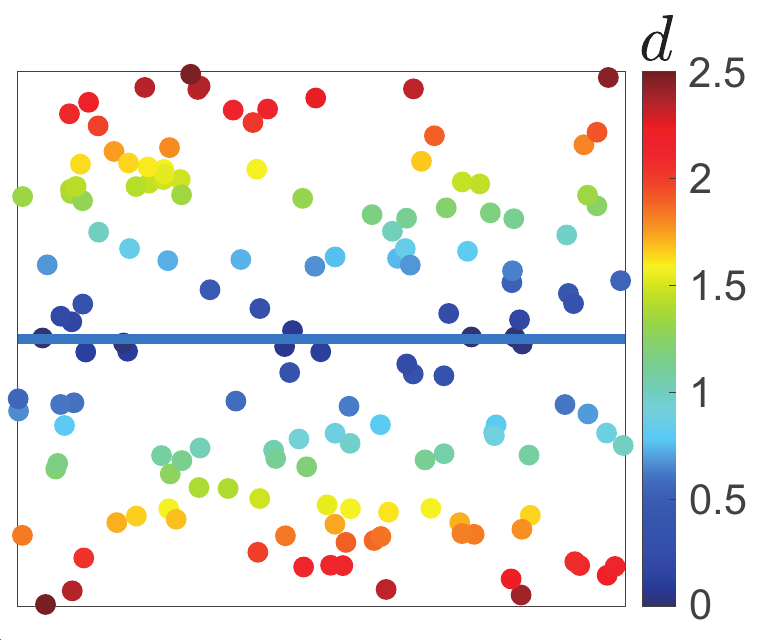}
	\label{fig5.1}
\end{minipage}
}
\subfloat[Weight]{
\begin{minipage}[t]{0.49\linewidth}
	\includegraphics[width=4.1cm,height=3.7cm]{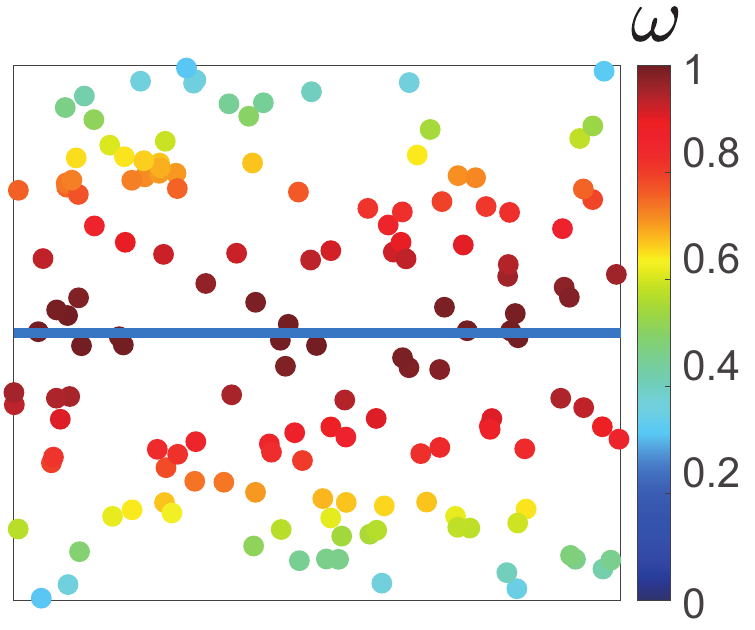}
	\label{fig5.2}
\end{minipage}
}
\caption{Weights are assigned to events based on event-line distances. (a) The colorbar represents the distance $ d $ from the event to the blue line. (b) The colorbar indicates the assigned weight $ \omega $ of events.}
\label{fig5}
\end{figure}

The pose parameters $\mathbf{X}=\left\{ \mathbf{R},\mathbf{T} \right\}$ can be obtained by ${\mathbf{X}^{*}}=\arg {{\min }_{\mathbf{X}}}C\left( \mathbf{X} \right)$, using non-linear optimization with the cost function:
\begin{equation}
C\left( \mathbf{X} \right)=\sum\limits_{i,j}{{{\omega }_{ij}}\cdot \left( d_{ij}^{2} \right)}.
\label{eq4}
\end{equation}

Solving the minimization (\ref{eq4}) is complicated because the objective function is highly non-convex in the residual errors $d_{ij}$. Theoretically, the robust estimation regains robustness by substituting the quadratic cost in the least squares problem with a robust cost function $\rho \left( \cdot  \right)$:  
\begin{equation}
\sum\limits_{l}{\rho \left( {{r}_{l}}\left( {{x}_{l}},\vartheta  \right);\sigma \right)},
\label{eqm1}
\end{equation}
where $\vartheta $ is the parameter to be estimated, $\sigma$ is a scale estimator. ${x}_{l}$ is the $l$-th measurement, and ${{r}_{l}}\left( {{x}_{l}},\vartheta \right)$ is the residual error of $l$-th measurement. According to the minimization principle, we look for partial derivatives w.r.t. $\vartheta$ and set the derivative to zero to find the optimal parameter:

\begin{equation}
\sum\limits_{l}{\rho \left( {{r}_{l}}\left( {{x}_{l}},\vartheta  \right);\sigma  \right)}\frac{\partial \rho }{\partial \vartheta }=0.
\label{eqm2}
\end{equation}

\begin{assumption}
\label{Assumption1}
The cost function $\rho \left( \cdot  \right)$ satisfies the following assumptions~\cite{yohai1987high}:

(i) $\rho$ is real, symmetric, twice continuously differentiable.

(ii) $\rho (0)=0$, $\rho$ is strictly increasing on $\left[ 0,c  \right]$, where $c$ is a finite constant.
\end{assumption}

Here, we introduce three robust estimation methods to enhance robustness against noise and outliers and further optimize the object pose: M estimation, S estimation and MM estimation. 

{\bf M estimation.}
M estimation is a classical robust estimation of the maximum likelihood type~\cite{huber1973robust}. The least-squares cost is replaced by the function $\rho \left( \cdot  \right)$, which is less sensitive to noise with large residuals. 

For event-based pose optimization problems, we aim to minimize the residual function:
\begin{equation}
\underset{\mathbf{R},\mathbf{T}}{\mathop{\min }}\,\sum\limits_{i,j}{\rho \left( \frac{{{d}_{ij}}}{{{{\hat{\sigma }}}_{M}}} \right)},
\label{eqm4}
\end{equation}
where 
\begin{equation}
{{{{\hat{\sigma }}}_{M}}}=\frac{M A D}{0.6745}=\frac{\operatorname{median}\left|d_{ij}-\operatorname{median}\left(d_{ij}\right)\right|}{0.6745}.
\label{eqm5}
\end{equation}

For example, $\rho \left( \cdot  \right)$ can be a Huber, Cauchy, or Geman-McClure cost~\cite{black1996unification}. Here, we use Tukey’s bisquare objective function. 
Let ${{u}_{ij}}={{{d}_{ij}}}/{{{{\hat{\sigma }}}_{M}}}$, and we have:
\begin{equation}
\rho \left( {{u}_{ij}} \right)\!=\!\left\{ \begin{array}{*{35}{l}}
\frac{u_{ij}^{2}}{2}-\frac{u_{ij}^{4}}{2{{c}^{2}}}+\frac{u_{ij}^{6}}{6{{c}^{4}}}& ,\left| {{u}_{ij}} \right|\le c  \\
\frac{{{c}^{2}}}{6} & ,\left| {{u}_{ij}} \right|>c  \\
\end{array} .\right.
\label{eq8}
\end{equation}

The closed-form solution for the weight function is:
\begin{equation}
\omega _{ij}=\left\{ \begin{array}{*{35}{l}}
{{\left[ 1-{{\left( \frac{{{u}_{ij}}}{c} \right)}^{2}} \right]}^{2}} & ,\left| {{u}_{ij}} \right|\le c  \\
0 & ,\left| {{u}_{ij}} \right|>{c}  \\
\end{array} ,\right.
\label{eqm6}
\end{equation}
where $c=4.685$ when employing Tukey’s bisquare weight function. 

{\bf S estimation.}
S estimation is based on the residual scale of M-estimator~\cite{rousseeuw1984robust}. However, M estimation does not fully consider the data distribution and uses the median as the weight value. To address the limitations of using the median, S estimation uses the residual standard deviation to determine the minimum robust scale estimator $\hat{\sigma }_{S} $, which satisfies the following condition:
\begin{equation}
\underset{\mathbf{R},\mathbf{T}}{\mathop{\min }}\,\sum\limits_{i,j}{\rho \left( \frac{{{d}_{ij}}}{{{{\hat{\sigma }}}_{S}}} \right)}.
\label{eqs1}
\end{equation}

First, the scale estimator $\hat{\sigma }_{S} $ is initialized according to Eq. (\ref{eqm5}), just in the same way than $\hat{\sigma }_{M} $. By setting the derivative Eq. (\ref{eqm2}) to zero, we derive the influence function $\psi$ :
\begin{equation}
\psi\left(u_{ij}\right)=\rho^{\prime}\left(u_{ij}\right)= \begin{cases}u_{ij}\left[1-\left(\frac{u_{ij}}{c}\right)^2\right]^2 & ,\left|u_{ij}\right| \leq c \\ 0& ,\left|u_{ij}\right|>{c} \end{cases}.
\label{eqs2}
\end{equation}

The weight $\omega _{ij}^{0}$ is initialized similarly to Eq. (\ref{eqm6}) with $c=1.547$~\cite{susanti2014m}. $\omega _{ij}^{0}$ is used as an initial guess for the subsequent iteration, to avoid being trapped in local minima.
After the initialization, the scale estimator $\hat{\sigma }_{S} $ and $ \omega_{ij} $ are updated at each iteration:
\begin{equation}
{{{\hat{\sigma }}}_{S}}=\sqrt{\frac{1}{0.199mn}\sum\limits_{j=1}^{m}{\sum\limits_{i=1}^{n}{{{\omega }_{ij}}}\cdot d_{ij}^{2}}},
\label{eqs4}
\end{equation}
\begin{equation}
{\omega _{ij}}=\frac{\rho ({{u}_{ij}})}{u_{ij}^{2}}.
\label{eqs5}
\end{equation}

\begin{figure}[t]
\centering
\removelatexerror
\begin{algorithm}[H]
	\caption{The robust pose optimization combined with M-estimation, S-estimation, and MM-estimation}
	\LinesNumbered
	\KwIn{Events $\left\{ {{e}_{i}} \right\}_{i=1}^{N}$;Intrinsic matrix $ \mathbf K $;\\ \qquad \quad Initial pose $\mathbf X^{(0)}$; Object 3D lines ${\mathbf {L}_{j}}$; \\ \qquad \quad MaxIterations; Threshold.}
	\KwOut{Optimal pose {$\mathbf X$}.}
	$d_{ij}^{(0)}=VariableInitial\left( {{e}_{i}},\mathbf K,{\mathbf {X}^{(0)}},{\mathbf {L}_{j}} \right)$;
	
	${{{\hat{\sigma }}}^{(0)}}=VariableInitial\left( d_{ij}^{(0)} \right)$;
	
	$u_{ij}^{(0)}=d_{ij}^{(0)}/{{\hat{\sigma }}^{(0)}}$ ;
	
	$\omega _{ij}^{(0)}=VariableInitial\left( u_{ij}^{(0)} \right)$ ;
	
	\For{$k=1,...,$ \rm{MaxIterations}}{
		$d_{ij}^{(k)}=VariableUpdate\left( {{e}_{i}},\mathbf K,{\mathbf {X}^{(k-1)}},{\mathbf{L}_{j}} \right)$;
		
		${{\hat{\sigma }}^{(k)}}=VariableUpdate\left( \omega _{ij}^{(k-1)},d_{ij}^{(k)} \right)$ ;
		
		${u_{ij}^{(k)}}=d_{ij}^{(k)}/{{\hat{\sigma }}^{(k)}}$ ;
		
		$\omega _{ij}^{(k)}=WeightUpdate\left( u_{ij}^{(k)} \right)$ ;
		
		$C\left( {\mathbf{X}^{\left( k-1 \right)}} \right)=\sum\limits_{i,j}{\omega _{ij}^{(k)}\cdot {{\left( d_{ij}^{(k)} \right)}^{2}}}$ ;
		
		\eIf{$\left\| \nabla C\left( {\mathbf{X}^{\left( k-1 \right)}} \right) \right\|<$ \rm{Threshold}}{
			break;
		}{
			${\mathbf{X}^{(k)}}={{\min }} C\left( {\mathbf{X}^{\left( k-1 \right)}} \right)$;
		}
	}
	\Return $\left(\mathbf {X} \right)$
\end{algorithm}
\end{figure}

{\bf MM estimation.}
MM estimation is a special M estimator~\cite{yohai1987high}, which first estimates the regression parameter using S estimation, and then performs M estimation using these initial estimates. The regression MM estimator is particularly effective when dealing with errors with a normal distribution. Let ${{\rho }_{1}}$ be another function that satisfies Assumption 1, and ${{\psi }_{1}}={{{\rho }'}_{1}}$. Then the MM estimator is the solution of 
\begin{equation}
\sum\limits_{i,j}{{{\psi }_{1}}\left( \frac{{{d}_{ij}}}{{{{\hat{\sigma }}}_{MM}}} \right){{\mathbf{e}}_{i}}=}\sum\limits_{i,j}{{{{{\rho }'}}_{1}}\left( \frac{{{d}_{ij}}}{{{{\hat{\sigma }}}_{MM}}} \right){{\mathbf{e}}_{i}}=0},
\label{eqmm1}
\end{equation}
where ${{{\hat{\sigma }}}_{MM}}$ is the standard deviation obtained from the residual of S estimation, and $ \rho $ is Tukey's biweight function:
\begin{equation}
\rho \left( {{u}_{ij}} \right)=\left\{ \begin{array}{*{35}{l}}
\frac{u_{ij}^{2}}{2}-\frac{u_{ij}^{4}}{2{{c}^{2}}}+\frac{u_{ij}^{6}}{6{{c}^{2}}} & ,\left| {{u}_{ij}} \right|\le c  \\
\frac{{{c}^{2}}}{6} & ,\left| {{u}_{ij}} \right|>c \\
\end{array} .\right.
\label{eqmm2}
\end{equation}

\begin{figure*}[t]
\centering
\subfloat[Median and mean errors w.r.t varying noise]{
\begin{minipage}[t]{1\linewidth}
	\centering
	\includegraphics[width=0.9\linewidth]{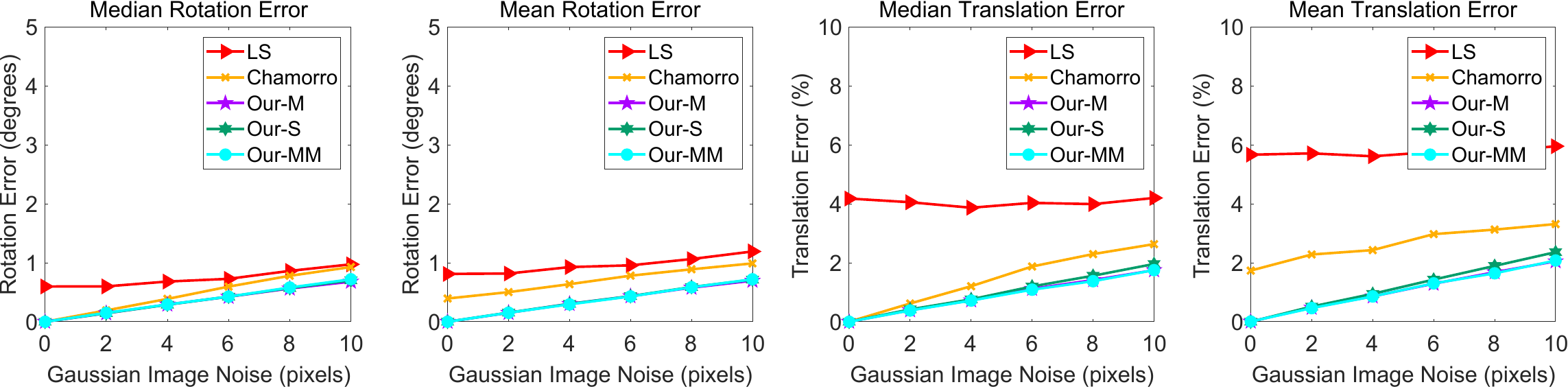}
	\label{fig6.1}
\end{minipage}%
}

\subfloat[Median and mean errors w.r.t varying outlier ratios]{
\begin{minipage}[t]{1\linewidth}
	\centering
	\includegraphics[width=0.9\linewidth]{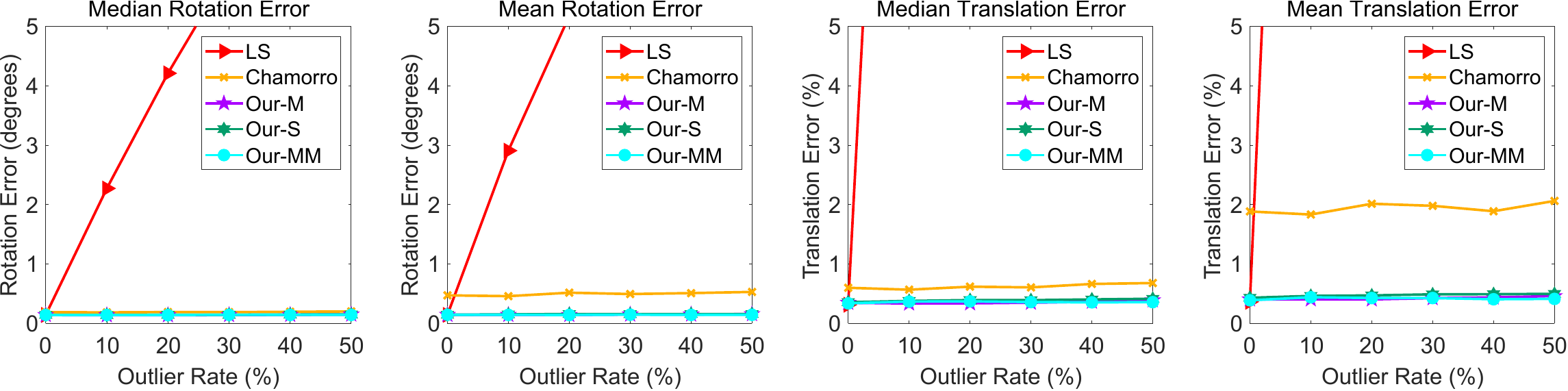}
	\label{fig6.2}
\end{minipage}%
}

\subfloat[Median and mean errors w.r.t varying number of lines]{
\begin{minipage}[t]{1\linewidth}
	\centering
	\includegraphics[width=0.9\linewidth]{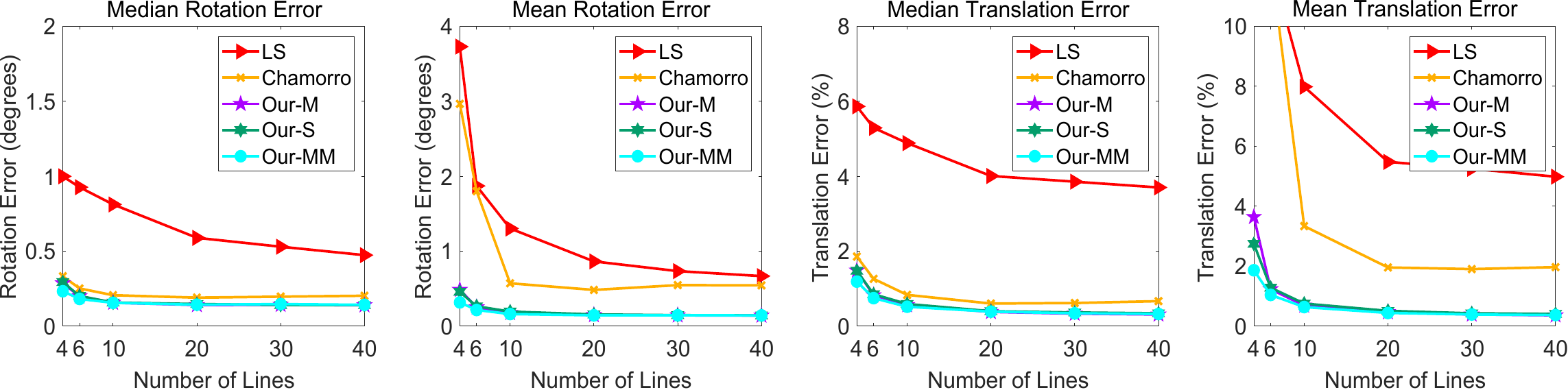}
	\label{fig6.3}
\end{minipage}%
}
\caption{Synthetic data experiment results. (a) Pose error w.r.t varying noise $\sigma$ from 0 to 10 pixels with the number of lines $n=25$ and outlier rate 2 \%. (b) Pose error w.r.t varying outlier ratios from 0 to 50 \% with the number of lines $n=25$ and noise $\sigma=2$ pixels. (c) Pose error w.r.t varying number of lines $n$ from 4 to 40 with fixed noise $\sigma=2$ pixels and outlier rate 2 \%.}
\label{fig6}
\end{figure*}

For each iteration, the MM-based weight value $\omega_{ij}$ updates using Eq. (\ref{eqm6}). The algorithm starts by solving for initial values of the parameters, based on the initial pose $\mathbf X^{(0)}$. Weights are then determined according to the magnitude of the residuals, and iterations are repeated to refine the weighting coefficients until the change of these coefficients is less than a certain allowable error. This refinement process is essential as it mitigates the impact of noise and further improves the robustness of the methods. The robust pose optimization methods combined with M-estimation, S-estimation, and MM-estimation, are summarized in Algorithm 1.

Once the object pose at time $t_i$ is determined, it is combined with the constant-velocity motion model to serve as the initial pose at time $t_{i+1}$. By establishing associations between events and lines using the event-line matching, the robust pose optimization method is then performed to minimize the event-line distances, thereby achieving the continuous pose tracking of objects.

\section{Experimental Evaluation}
\label{sec:experi}
In this section, we conduct experiments with synthetic data, simulated events, and real events, respectively. The proposed methods are compared with state-of-the-art methods, demonstrating their superior accuracy in both challenging and regular scenarios.

\subsection{Synthetic Data Experiments}
\label{experi1}
To generate synthetic data, we are given a virtual event camera with image size $640\times 480$ pixels, focal length $800$ pixels, and the principal point at the image center. The 3D lines are randomly generated and their depth is randomly distributed in the range $ [5,10]$ meters in the camera coordinate frame. Then, these 3D lines are projected on the event plane using the ground truth of the rotation matrix ${{\mathbf{R}}_{truth}}$ and translation vector ${{\mathbf{T}}_{truth}}$. Events are generated randomly near the projected lines. 

To ensure a fair comparison between methods, we provide the same initial value by perturbing the ground truth~\cite{kneip2014opengv}, and carry out a total of 1000 trials in the synthetic experiment. The error metric of rotation matrix $\mathbf{R}$ and translation vector $\mathbf{T}$ are defined as~\cite{7139836,xu2016pose}: 
$Er{r_{\bf{R}}} = \arccos \left( {\left( {\text{Trace}\left( {{{\bf{R}}^ \top }{{\bf{R}}_{{\rm{truth}}}}} \right) - 1} \right)/2} \right)$ and 
$Er{{r}_{\mathbf{T}}}=\left\| {\mathbf{T}-{\mathbf{T}_{\text{truth}}}} \right\|/\left\| {{\mathbf{T}_{\text{truth}}}} \right\|$.

Since no open-source event-based object pose estimation project is yet available, we implement two baseline approaches: \textit{(i)} LS-based method is the traditional nonlinear least square method, that endows each event the same weight for pose optimization. \textit{(ii)} Chamorro \emph{et al.}~\cite{chamorro2022event} estimate and track the 6-DoF pose of the event camera, which is an event-based line-SLAM method. For comparison, we adopt the key techniques of the method and apply them for object pose estimation and tracking. 
Our method combines M estimation, S estimation, and MM estimation, referred to as Our-M, Our-S, and Our-MM, respectively.

\textbf{Accuracy w.r.t varying noise.}
The first experiment is designed to evaluate the accuracy of our method with respect to the increasing noise. The standard deviation of Gaussian noise is added to events, ranging from 0 to 10 pixels in steps of 2 pixels. The number of lines is set to $n=25$, and the outlier rate is 2\%. The median and mean errors of the pose estimation are shown in Fig.~\ref{fig6}\subref{fig6.1}, which almost grows linearly as the noise increases. The experimental results demonstrate that our three methods exhibit comparable levels of accuracy, surpassing those of other methods. The observation suggests that our robust estimation methods are effectively resistant to the effects of Gaussian noise.

\textbf{Accuracy w.r.t varying outlier rate.} 
In this scenario, the outliers are incorrect correspondences generated by randomly selecting correspondences between the lines and events. The outlier rate varies from 0\% to 50\% in steps of 10\%, the number of lines is set to $n=25$, and the Gaussian noise is fixed to $\sigma=2$ pixels. As can be seen in Fig.~\ref{fig6}\subref{fig6.2}, the LS-based method is not robust to outliers, as it includes all events in its estimation process, even those that may be outliers. Our methods achieve satisfactory performance even with the presence of a reasonable amount of outliers. Among them, Our-MM performs the best with the least error. This is because robust estimation algorithms are used, which effectively remove outliers or reduce their influence. 

\textbf{Accuracy w.r.t varying number of lines.} 
In the third experiment, we investigate the accuracy with a fixed outlier rate of 2\% and noise $\sigma=2$ pixels, varying the number of lines from 4 to 40. As indicated in Fig.~\ref{fig6}\subref{fig6.3}, the median and mean errors of the rotation matrix and translation vector gradually stabilize as the number of lines increases. Once the number of lines exceeds ten, the errors remain relatively stable. Furthermore, our methods have demonstrated faster convergence, as well as higher accuracy and reliability compared to other methods.
\begin{figure*}[t]
\centering
\subfloat[High-speed motion.]{
	\begin{minipage}[t]{1\linewidth}
		\includegraphics[width=17cm]{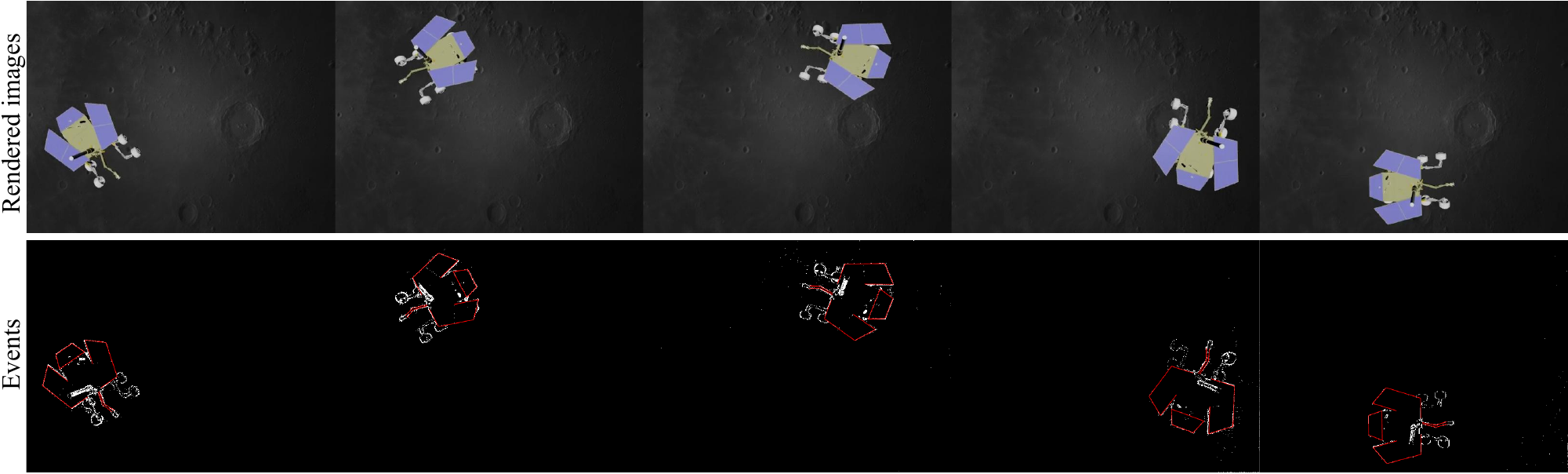}
		\label{fig7a}
	\end{minipage}%
}

\subfloat[HDR scenes.]{
	\begin{minipage}[t]{1\linewidth}
		\includegraphics[width=17cm]{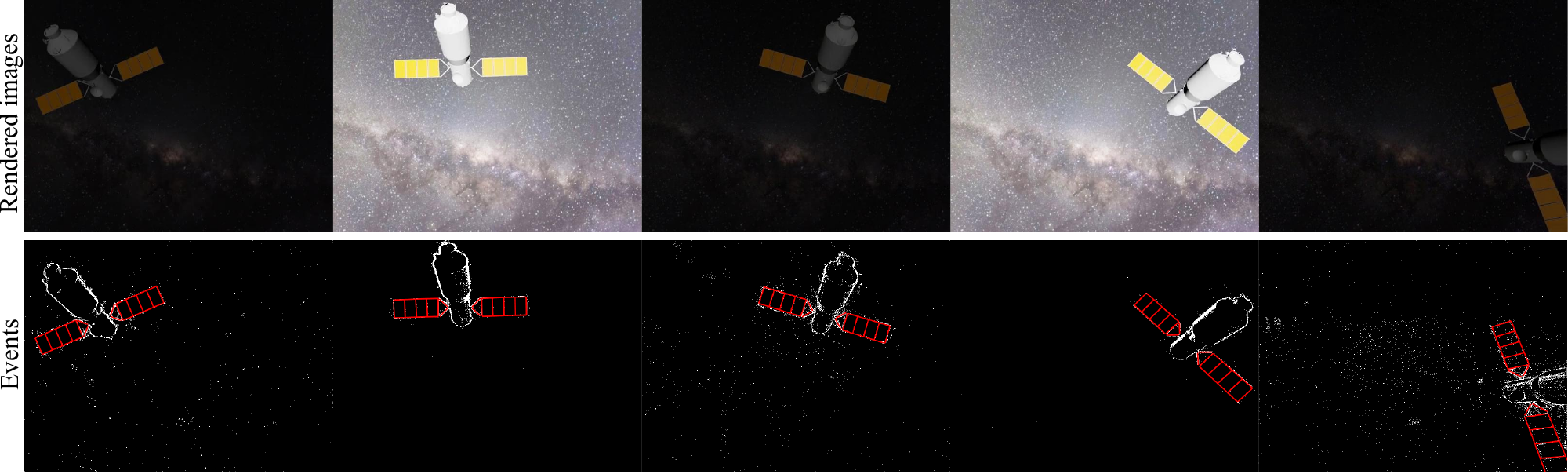}
		\label{fig7b}
	\end{minipage}%
}
\caption{The simulated event experiment results of high-speed motions and HDR scenes. Above is the rendered images of the object motion, below is the generated events and pose tracking results. 3D lines of the object model are reprojected onto the event planes (for visualization purpose only) using Our-MM and Our-S, indicated by the red lines.}
\label{fig7}
\end{figure*}

\begin{table*}[htbp]
\centering
\caption{Experimental setup and pose tracking error of high-speed motions and HDR scenes.}
\setlength{\tabcolsep}{6pt}
\begin{tabular}{ccccccccccc}
	\toprule
	\multirow{4}[1]{*}{Test scenarios}&\multicolumn{5}{c}{Experimental Parameters}&\multicolumn{5}{c}{Absolute Trajectory Error (unit-less)}\\
	\cmidrule(r){2-6}\cmidrule(r){7-11}
	&Description&\thead{Time\\(s)}&\thead{EPS\\(event/s)}&\thead{Average\\speed (m/s)}&\thead{Average angular\\speed (rad/s)}&LS-based&Chamorro&Our-M&Our-S&Our-MM\\
	\midrule
	\multirow{3}[0]{*}{\thead{High-speed\\motions}}&Normal&20.00&$3.43\times{10^5}$&0.71&0.32&1.89&1.45&0.68&0.45&\bf{0.36}\\
	&Fast&10.00&$6.72\times{10^5}$&1.41&0.63&1.79&1.52&0.71&0.43&\bf{0.37}\\
	&Faster&5.00&$4.21\times{10^6}$&2.83&1.26&2.45&1.94&0.77&0.67&\bf{0.49}\\
	\midrule
	\multirow{2}[0]{*}{HDRscenes}&Constantlight&10.00&$5.47\times{10^5}$&0.77&0.18&3.58&2.86&2.96&2.78&\bf{2.57}\\
	&Changinglight&10.00&$1.32\times{10^6}$&0.77&0.18&6.65&5.51&5.13&\bf{4.82}&5.04\\
	\bottomrule
\end{tabular}
\label{tab0}
\end{table*}

\begin{figure}[htbp] 
\centering{\includegraphics[width=8cm]{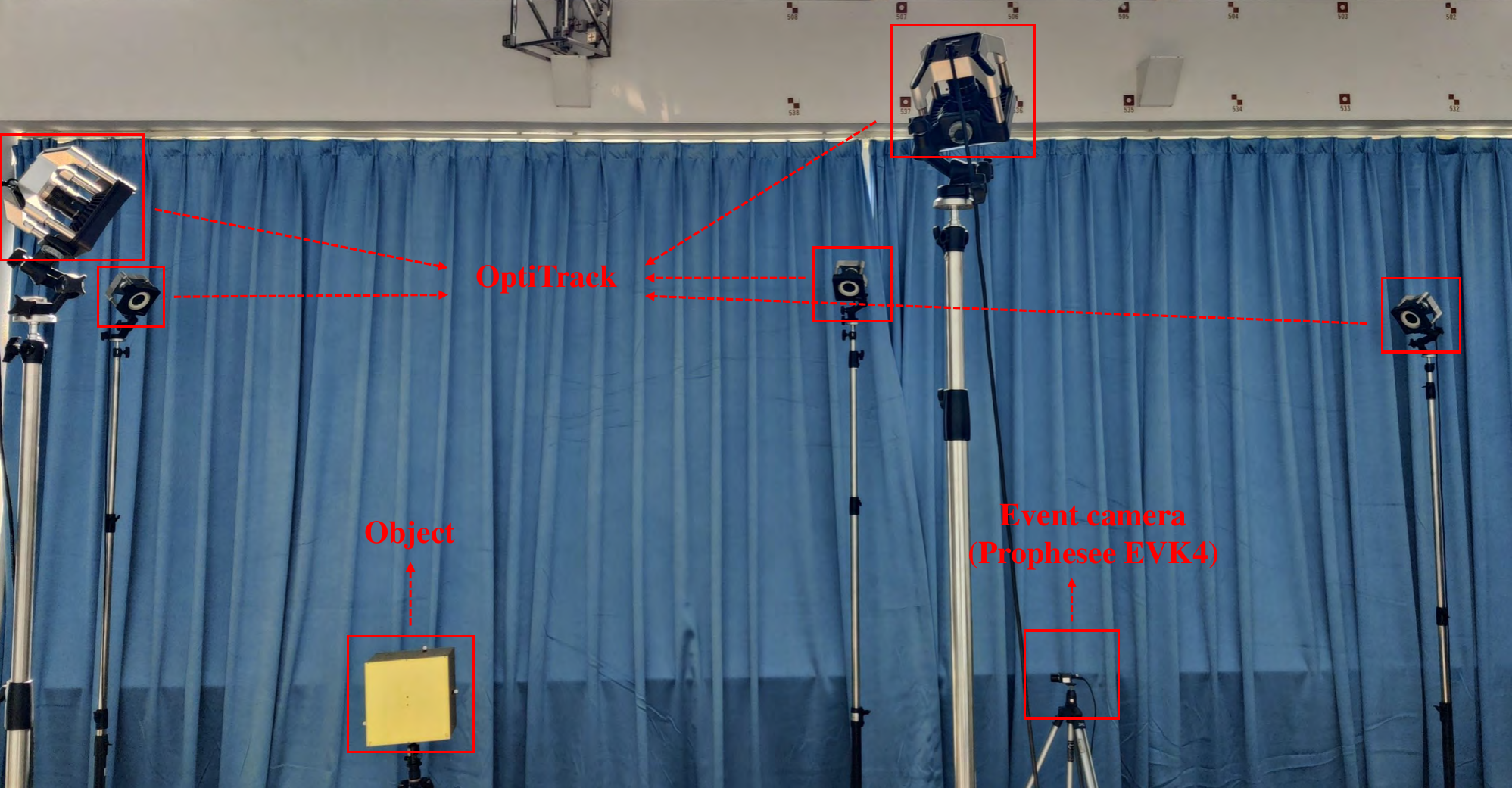}}
\caption{Experimental scene layout for real-world object pose estimation and tracking.}
\label{figExpe_scen}
\end{figure}

\begin{figure*}[htbp]
\centering{\includegraphics[width=17cm]{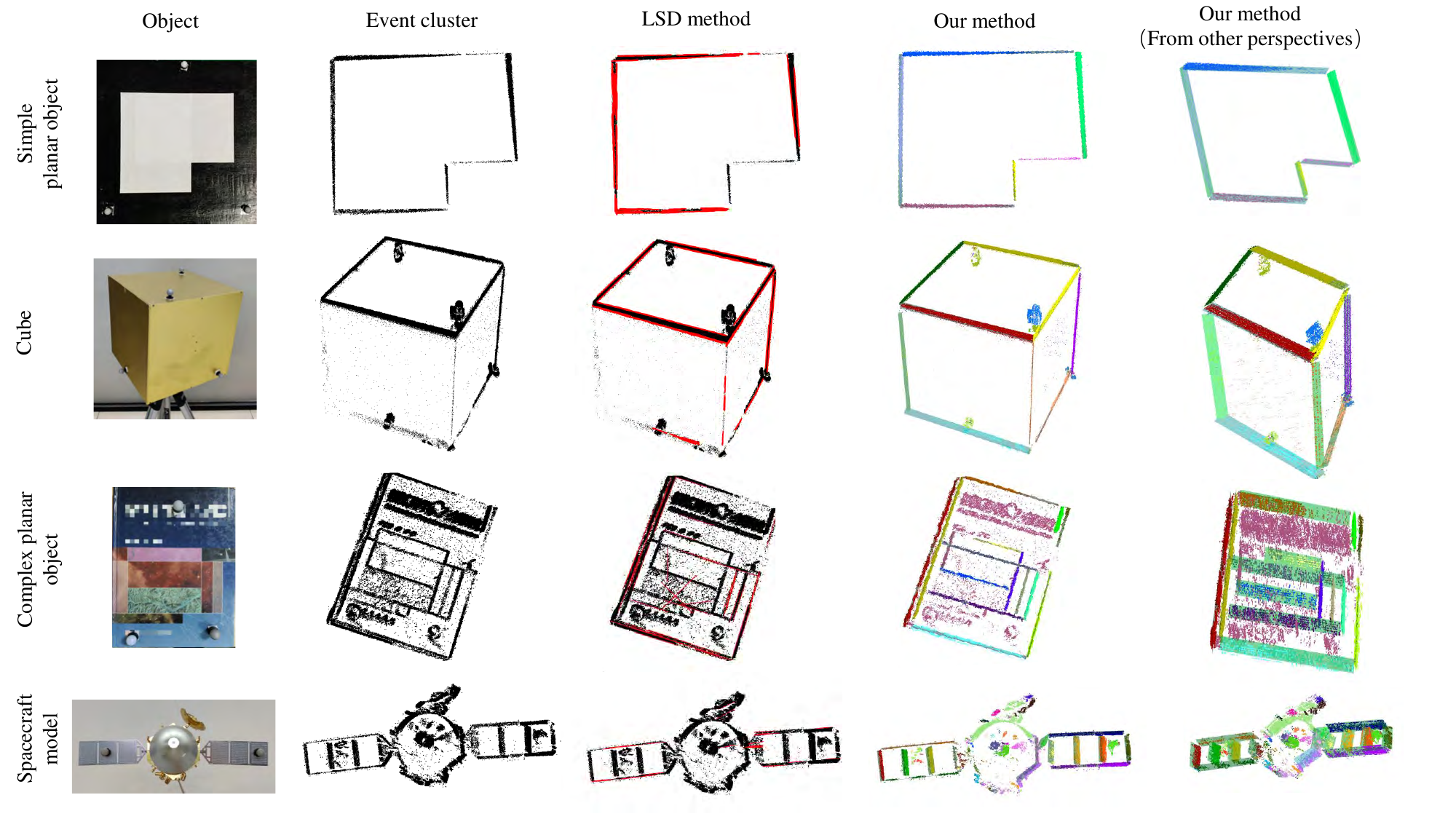}}
\caption{Line detection results of the real data. From left to right: object, event cluster, LSD method, our method (from different perspectives). The results of segmentation and plane fitting are depicted in various colors. The edges of the planes denote the detected lines.}
\label{fig8}
\end{figure*}

\subsection{Simulated event experiment}
\label{experi2}
Due to the rapid development of event cameras, there are many event-based SLAM datasets available. However, there is still a lack of event-based datasets specifically designed for object pose estimation and tracking. To evaluate the precision of the proposed methods in object pose estimation and tracking, we have devised and established an event-based dataset for moving objects, including simulated and real events.

To thoroughly validate the method's feasibility in challenging scenarios, such as high-speed motions and HDR scenes, we conduct simulated event experiments. This entails rendering motion videos of objects using Blender and then transforming the videos into events. Initially, two spacecraft are selected as test objects, each assigned distinct trajectories and velocities. The movements of objects are captured by a stationary monocular camera at a frame rate of 30 Hz and a resolution of $640\times480$ pixels. Then, rendered RGB videos are transformed into event streams using V2E~\cite{9523069}, with the relevant parameters set to be consistent with Blender. The rendered images of object motion are displayed in the upper section of Fig.~\ref{fig7}. To quantitatively assess the accuracy of our pose tracking method, we compare the estimated poses with ground truth and employ the toolbox from Grupp.~\cite{grupp2017evo} to compute the Root Mean Square Error (RMSE) of the absolute trajectory error (ATE).

\textbf{High-speed motion}: To validate the efficacy of our method in tracking high-speed moving objects, we categorize the sequences based on the velocity of object motions: normal, fast, and faster, as shown in Table~\ref{tab0}. The precision of all methods diminishes as velocity increases. Our methods demonstrate robust pose tracking capabilities, as depicted in Fig.~\ref{fig7}\subref{fig7a}, with Our-MM exhibiting the minimal margin of error.

\textbf{HDR scenes}: We introduce rapid changes in illumination during the object movement. Additionally, we establish a control experiment devoid of illumination variations for comparative analysis. The pose tracking errors are quantitatively displayed in Table~\ref{tab0}, and the visualization of the tracking results is depicted in Fig.~\ref{fig7}\subref{fig7b}. Even in HDR scenarios, there is a lot of event noise. Our methods effectively establish associations between events and lines, thereby enabling accurate object tracking.

The experimental results of simulated events demonstrate that our method is capable of handling challenging scenarios, including high-speed motions and HDR scenes. This can be attributed to the advantages of event cameras, which possess a high dynamic range and without motion blur, in contrast to traditional cameras. More importantly, our method effectively leverages the temporal and spatial characteristics of events to establish associations between events and lines, thereby enabling the continuous and stable object tracking, even in challenging scenarios.

\subsection{Real Data Experiments}
\label{experi3}
In this section, we conduct pose tracking experiments using real events. Firstly, we choose the Prophesee EVK4 event camera  ($ 1280\times 720$ pixels), and the OptiTrack motion capture system, both of which have been pre-calibrated. The experimental scene layout is shown in Fig.~\ref{figExpe_scen}. Then, we select several representative objects, including a simple planar object and a non-planar object (cube), which are pretty common in man-made environments. The former consists of merely six planar lines, while the latter encompasses twelve non-planar lines. To further increase the difficulty, we prepare more complex objects with a greater number of lines: a planar object (book) and a non-planar object (spacecraft model). These two objects contain a greater abundance of lines, which are densely distributed. Consequently, this results in an increased occurrence of disordered events, posing challenges for pose estimation and tracking. The event camera is placed at a fixed distance from the objects, and we move the objects while the camera captures their motion and records the corresponding events. The 3D models of these objects are known in advance. Markers are attached to these objects to obtain the ground truth of their poses and trajectories by OptiTrack. 

To further enhance the quality of our real event datasets, we have added numerous challenging conditions, e.g., rapidly changing illumination, noise, reflections, and partial occlusion. Notably, the flashing LED lights of OptiTrack bring additional noise, which is undoubtedly a great challenge for pose estimation. The surface of the cube exhibits diffuse reflection, while the surface of the spacecraft model exhibits specular reflection, resulting in the occurrence of interfering events. Furthermore, our dataset takes into account the partial occlusion of objects, which persists as a significant issue warranting attention.
\begin{table*}[htbp]
\centering
\caption{Quantitative comparison of the initial pose estimation with the state-of-the-art methods.}
\label{tab1}
\begin{tabular}{ccccccc}
\toprule
\multirow{2}[3]{*}{Input}&\multirow{2}[3]{*}{Correspondences}&\multirow{2}[3]{*}{Method}&\multicolumn{4}{c}{Mean reprojection error (unit: pixel)}\\
\cmidrule{4-7}& & &Simple planar object &Cube&Complex planar object&Spacecraft model\\
\midrule
2D/3D points&\multirow{2}[2]{*}{Known}&OPnP~\cite{zheng2013revisiting}&1.72&\textbf{1.54}&\textbf{1.77}&2.47\\
2D/3D lines& &ASPnL~\cite{xu2016pose}&\textbf{1.69}&2.28&3.51&2.62\\
2D/3D lines& &RoPnL~\cite{liu2020globally}&1.85&1.97&1.85&\textbf{2.28}\\	
\midrule
2D/3D points &\multirow{4}[2]{*}{Unknown}&SoftPOSIT~\cite{david2004softposit}&-&1.50&-&3.76\\
2D/3D points& &BlindPnP~\cite{moreno2008pose}&-&1.65&-&3.70\\
2D/3D lines & &RANSAC~\cite{fischler1981random}&2.13&1.54&3.16&3.09\\
2D/3D lines	& &Our&\textbf{1.96}&\textbf{1.44}&\textbf{2.73}&\textbf{2.82} \\
\bottomrule
\end{tabular}%
\end{table*}%

\begin{figure*}[htbp]
\centering
\subfloat[Simple planar object]{
\begin{minipage}[]{0.5\linewidth}
	\centerline{\includegraphics[width=7cm]{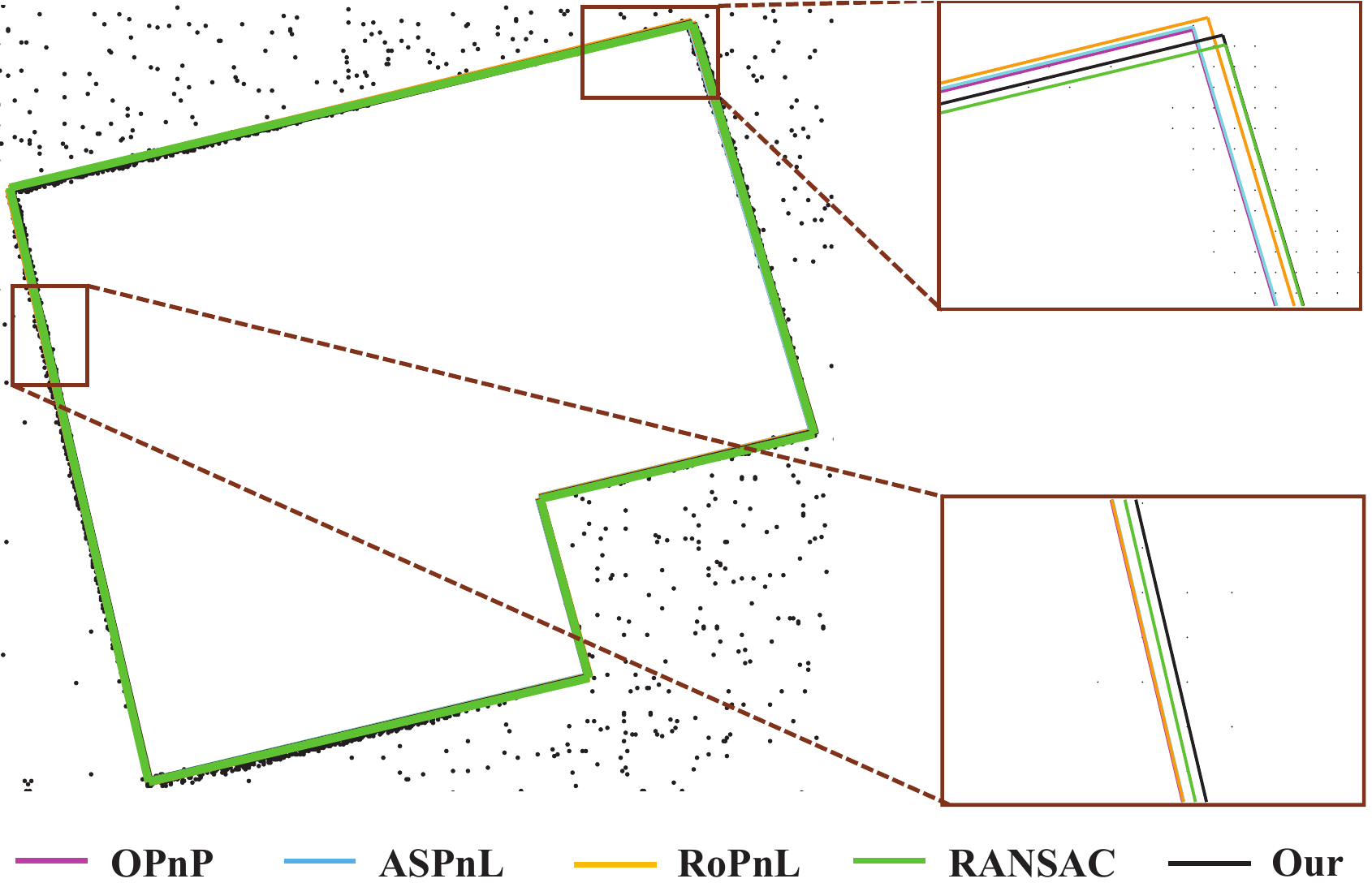}}
	\label{fig9.1}
\end{minipage}%
}
\subfloat[Complex planar object]{
\begin{minipage}[]{0.5\linewidth}
	\centering
	\centerline{\includegraphics[width=7cm]{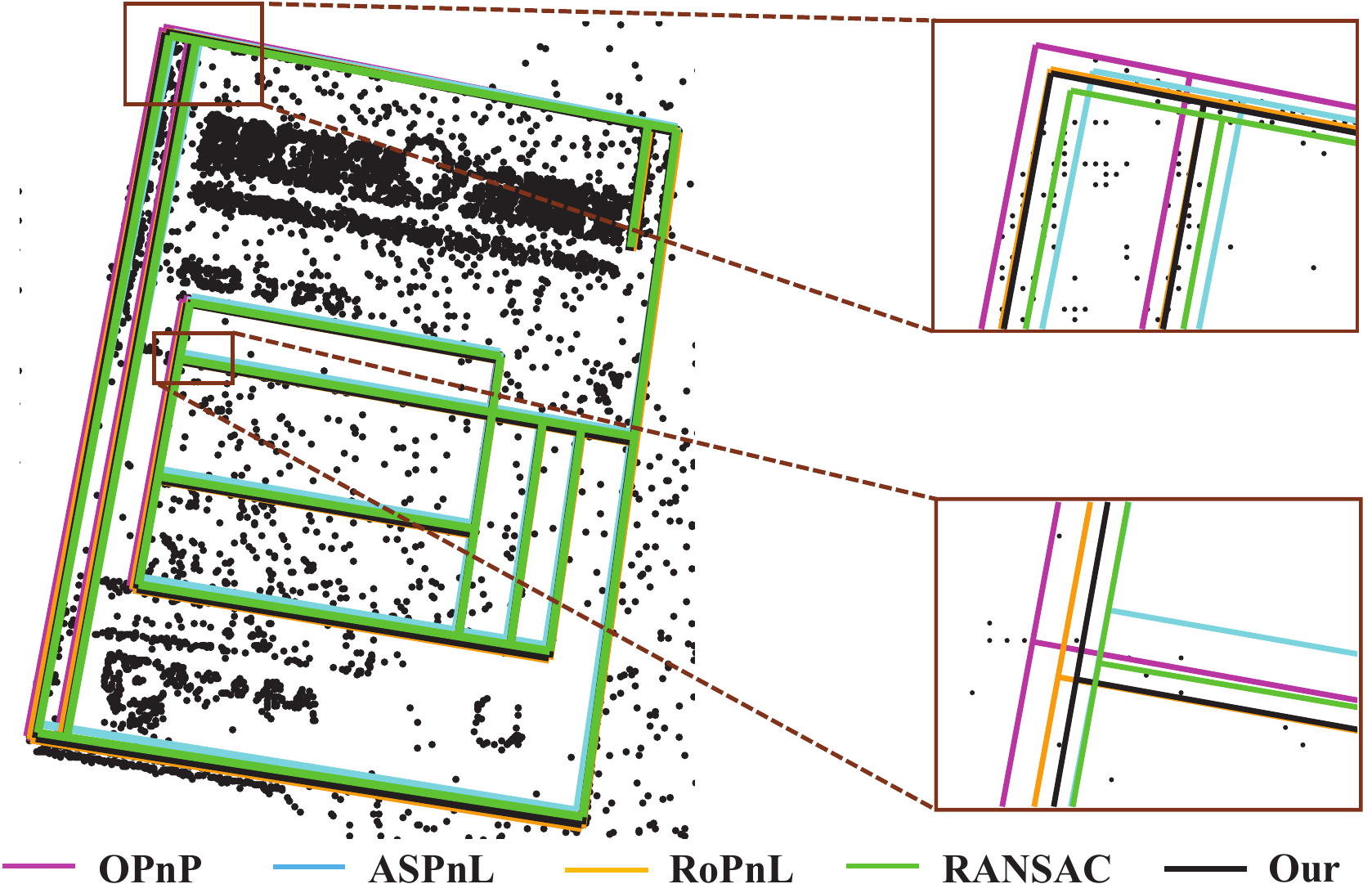}}
	\label{fig9.3}
\end{minipage}%

}
\\

\subfloat[Cube]{
\begin{minipage}[]{0.5\linewidth}
	\centerline{\includegraphics[width=7cm]{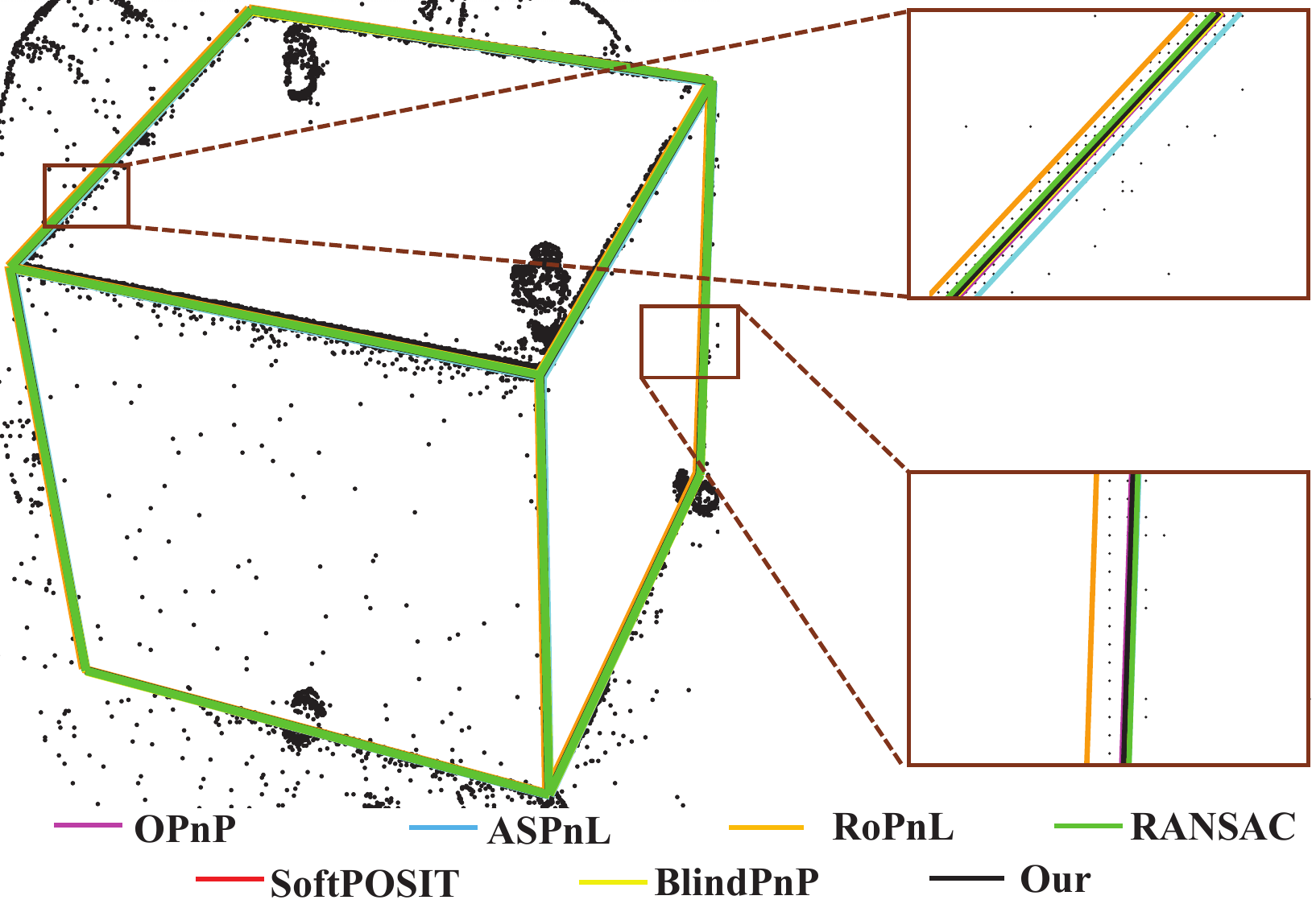}}
	\label{fig9.2}
\end{minipage}%
}
\subfloat[Spacecraft model]{
\begin{minipage}[]{0.5\linewidth}
	\centering
	\centerline{\includegraphics[width=7cm]{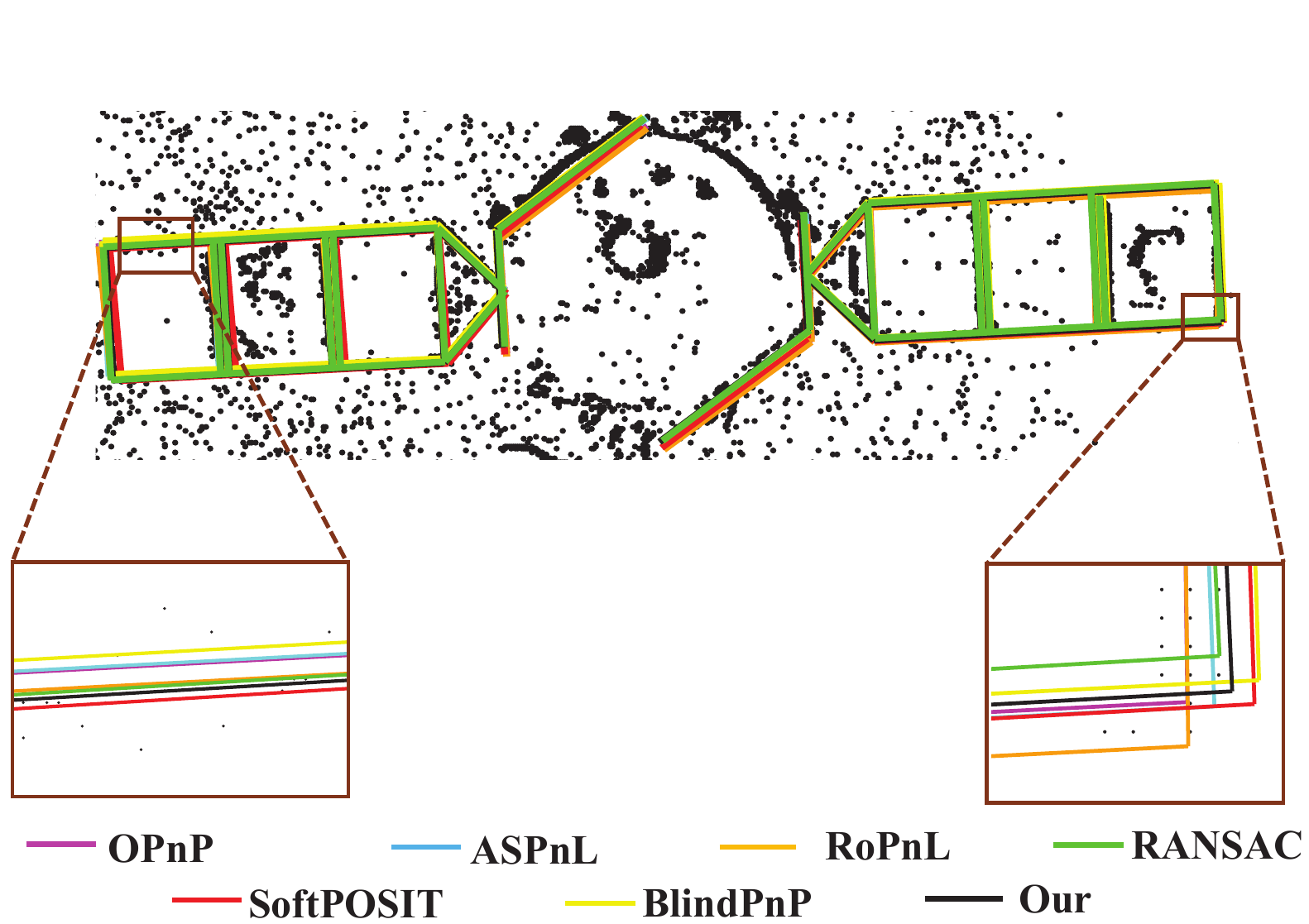}}
	\label{fig9.4}
\end{minipage}%
}
\caption{Qualitative comparison of the initial pose estimation. (a)-(d) respectively depict the reprojected results of a simple planar object, a complex planar object, a cube, and a spacecraft model. The black dots represent events, and the lines of objects are reprojected on the event planes (for visualization purpose only) using state-of-the-art methods.}
\label{fig9}
\end{figure*}

In order to demonstrate the performance of our method in real-world scenarios, we conduct multiple sets of experiments. 

Firstly, we cluster events and extract lines directly from the event clusters. The event cluster is reprojected onto the event plane for visualization, as shown in the second column of Fig.~\ref{fig8}. We adopt the classical LSD algorithm~\cite{von2012lsd} for comparison, which detects lines from the event plane. The LSD algorithm demonstrates remarkable efficacy on conventional images. However, its performance is compromised when applied to events due to the significant amount of noise present in real-world scenarios, often resulting in confusing results, as displayed in the third column of Fig.~\ref{fig8}. The algorithm detects multiple cluttered lines and there are cases of missed detections. Additionally, the detected lines are challenging to associate with time due to the accumulation of events on the event plane, leading to the loss of temporal information. 

Next, we directly extract lines from the event clusters. Firstly, the event cluster is transformed into point clouds in the 3D space-time volume, and undergoes denoising. We perform point cloud segmentation and plane fitting, which results in multiple planes. The intersection lines between the fitted and event planes represent the object lines. The results of point cloud segmentation and plane fitting from different perspectives are depicted in various colors, as presented in columns four and five of Fig.~\ref{fig8}. The edge lines correspond to the object lines at the initial and final times of the event cluster, respectively. By contrast, our method produces visually pleasing line detection results.

Then, we utilize the extracted lines to initialize the object pose. To verify the accuracy of the proposed pose initialization method, we compare it with other baseline methods: \textit{(i)} OPnP~\cite{zheng2013revisiting}, ASPnL~\cite{xu2016pose} and RoPnL~\cite{liu2020globally}. We provide a good quality set of 2D–3D point or line correspondences beforehand. \textit{(ii)} SoftPOSIT~\cite{david2004softposit}, BlindPnP~\cite{moreno2008pose}, and RANSAC~\cite{fischler1981random}. Given 2D-3D points or lines with unknown correspondences, these methods are adopted to estimate the initial pose. Among them, the first three methods require knowledge of the correspondences among points or lines, whereas the latter three do not. The RANSAC method involves setting the confidence level to 0.99 and randomly selecting 2D and 3D lines required for the ASPnL solver~\cite{xu2016pose} within the RANSAC framework. The initial pose yielding the maximum number of inliers is then adopted. The accuracy of the estimated pose is evaluated using the mean reprojection error of points (endpoints). Table~\ref{tab1} reports the quantitative comparison of our method with the state-of-the-art methods. Fig.~\ref{fig9} provides an intuitive visual representation of reprojection results using the initial pose. For the planar object, the size of the search space increases from the 6D pose space to the 8D space of homography, so that a naive extension of the SoftPOSIT and BlindPnP fails to converge. Furthermore, the various errors encountered in real experiments, such as line extraction, unavoidably impact the convergence of these two methods. By contrast, our method can provide a reliable initial pose for all tested objects with high accuracy, without requiring line correspondences or pose prior.
\begin{figure*}[htbp] 
\centerline{\includegraphics[width=18cm]{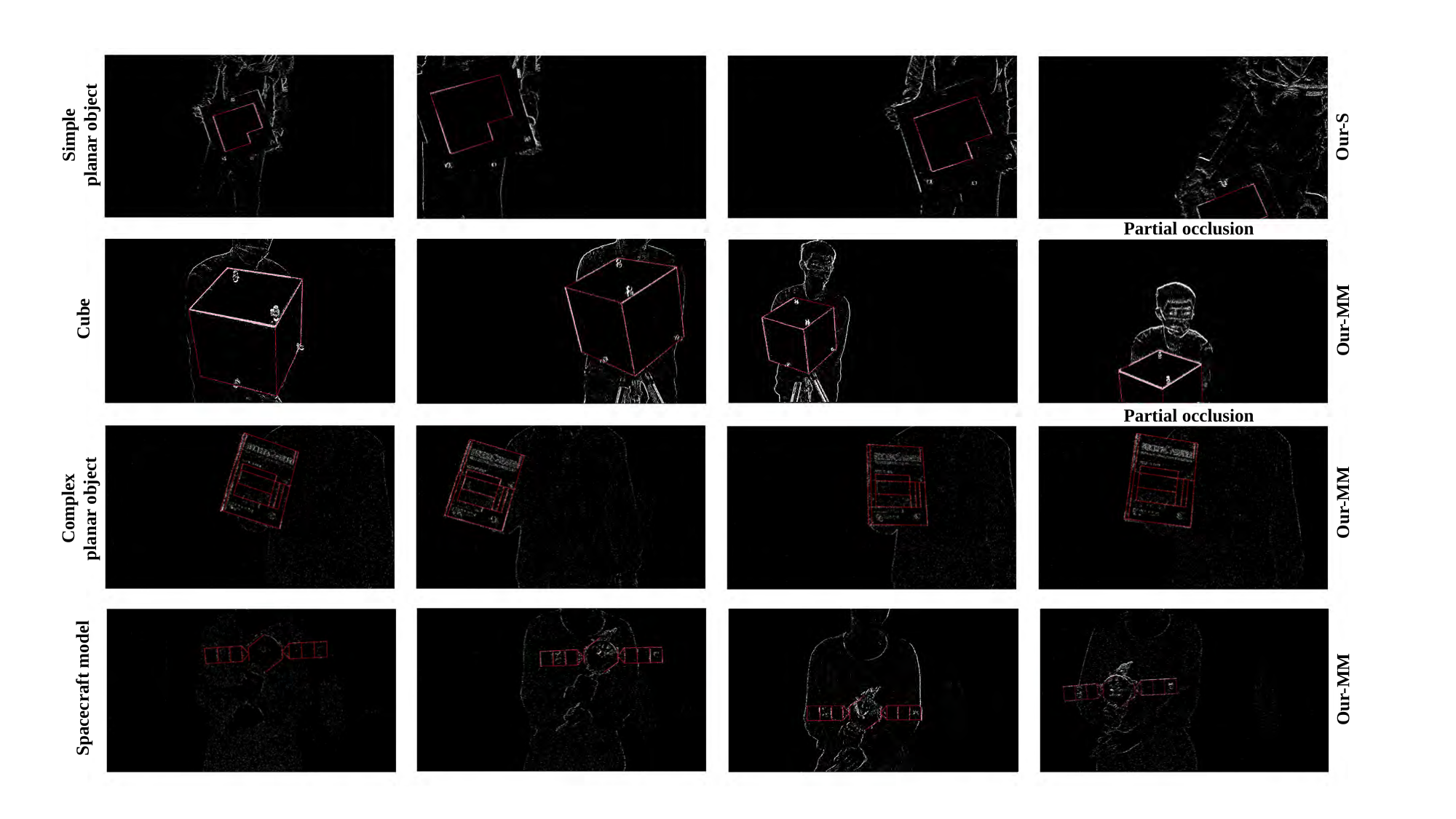}}
\caption{Visualization of the pose tracking results. 3D lines of the object model are reprojected onto the event planes (for visualization purpose only) using Our-S and Our-MM methods, indicated by the red lines.}
\label{fig10}
\end{figure*}

\begin{table*}[ht]
\centering
\caption{Quantitative comparison with the state-of-the-art methods in terms of Absolute Trajectory Error.}
\label{tab2}
\setlength{\tabcolsep}{7pt}
\newcommand{\tabincell}[2]{\begin{tabular}{@{}#1@{}}#2\end{tabular}}
\begin{tabular}{cccccccccccc}
\toprule
\multirow{2}[2]{*}{Object} & \multirow{2}[2]{*}{Method} & \multicolumn{4}{c}{Absolute Trajectory Error (unit-less)} & \multirow{2}[2]{*}{Object} & \multirow{2}[2]{*}{Method} & \multicolumn{4}{c}{Absolute Trajectory Error (unit-less)} \\
\cmidrule{3-6}\cmidrule{9-12}& &rmse&mean&median&std& & &rmse&mean&median&std\\
\midrule
\multirow{5}[2]{*}{\tabincell{c}{Simple planar object}}&LS-based&3.62&2.35&2.00&1.72&\multirow{5}[2]{*}{Cube}&LS-based&14.98&7.30&4.52&13.08\\
&Chamorro&3.45&2.42&2.03&1.74& &Chamorro&6.62&3.28&1.77&5.75\\
&Our-M&0.77&0.73&0.76&0.19& &Our-M&1.40&1.15&0.93&0.80\\
&Our-S&\textbf{0.56}&\textbf{0.55}&\textbf{0.53}&\textbf{0.14}& &Our-S&1.29&1.05&\textbf{0.84}&0.73\\
&Our-MM&0.59&0.57&\textbf{0.53}&0.19& &Our-MM&\textbf{1.12}&\textbf{0.99}&0.86&\textbf{0.53}\\
\midrule
\multirow{5}[2]{*}{\tabincell{c}{Complex planar object}}&LS-based&3.55&3.54&3.51&\textbf{0.18}& \multirow{5}[2]{*}{\tabincell{c}{Spacecraft model}}&LS-based&2.26&2.26&2.23&\textbf{0.10}\\
&Chamorro&3.50&3.48&3.45&0.23& &Chamorro&1.47&1.46&1.42&0.18\\
&Our-M&3.48&3.47&3.48&0.21& &Our-M&1.38&1.36&1.33&0.20\\
&Our-S&3.36&3.35&3.34&0.26& &Our-S&1.46&1.44&1.41&0.19\\
&Our-MM&\textbf{2.18}&\textbf{2.16}&\textbf{2.00}&0.33& &Our-MM&\textbf{1.36}&\textbf{1.35}&\textbf{1.31}&0.21\\
\bottomrule
\end{tabular}
\end{table*}

\begin{figure*}[ht]
\centering
\subfloat[]{
	\begin{minipage}[t]{0.72\linewidth}
		\centerline{\includegraphics[width=12.5cm,height=8.5cm]{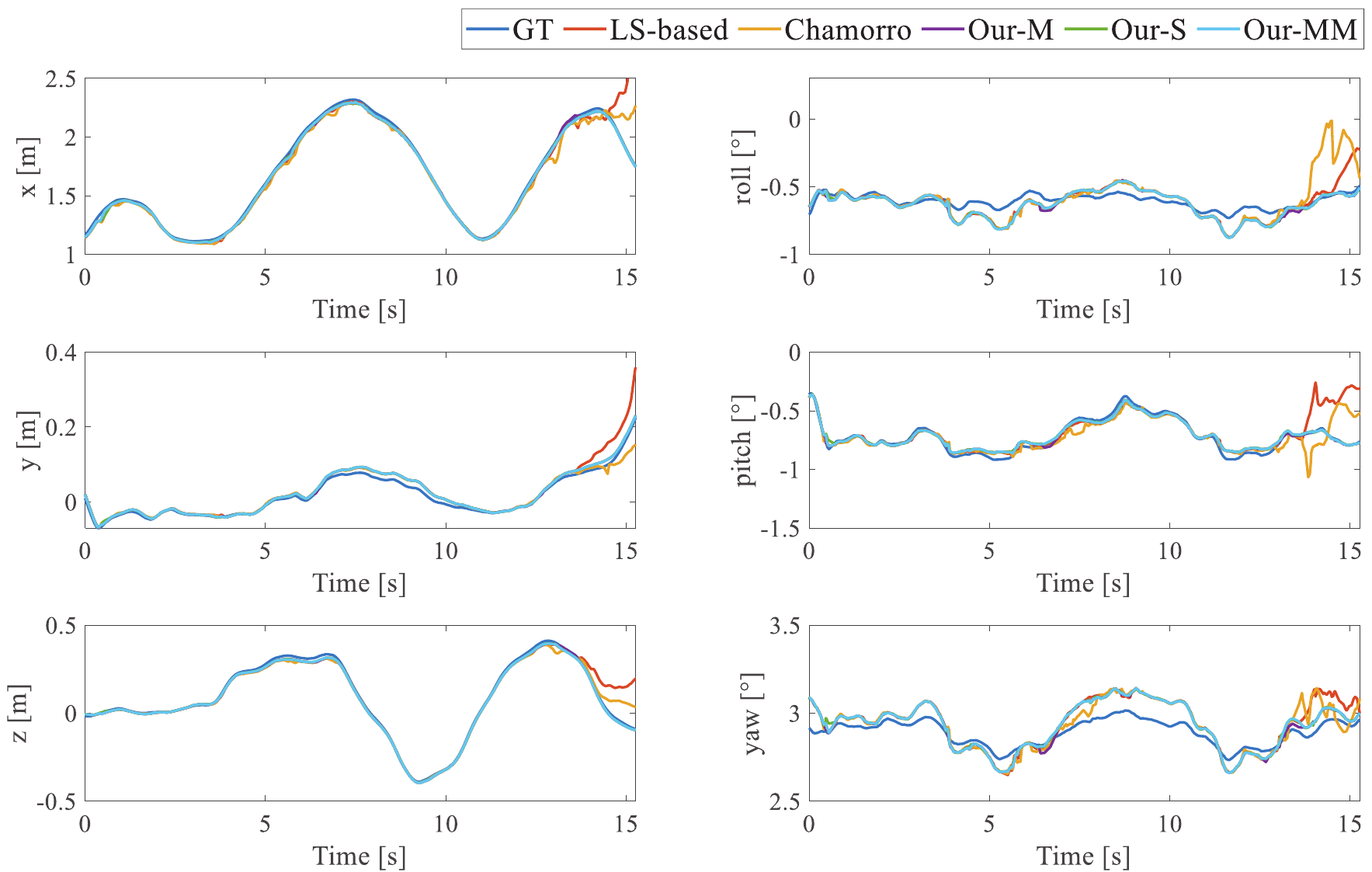}}
		\label{fig11.1}
	\end{minipage}%
}
\subfloat[]{
	\begin{minipage}[t]{0.25\linewidth}
		\centering
		\centerline{\includegraphics[width=5cm]{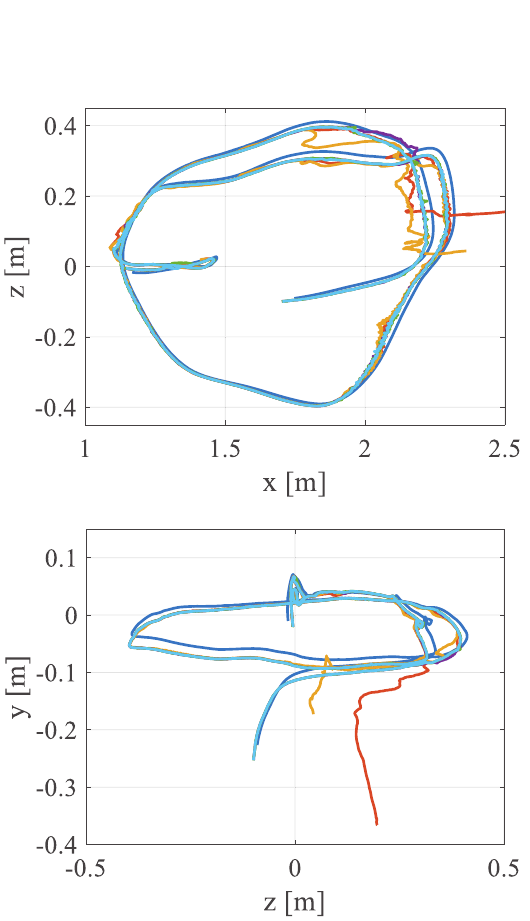}}
		\label{fig11.2}
	\end{minipage}%
}
\caption{Experimental results of pose estimation and tracking for the cube object. (a) Quantitative comparison with the state-of-the-art methods in terms of rotation and translation error. (b) Qualitative comparison of the object trajectories.}
\label{fig11}
\end{figure*}

\textbf{Simple object experiments.} 
We first conduct experiments to analyze the effect of our methods on simple objects (a planar object and a cube), both of which have a simple geometric structure. The distance thresholds are set as ${{d}_{t}}=8$ pixels and $ d_m=0.5d_l $, which is a good accuracy-efficiency compromise. To illustrate the robustness and accuracy of the methods, we evaluate the APE between the estimated poses and the ground truth. The error distributions of the pose tracking results are shown in Table~\ref{tab2}. It is evident that our methods perform better than other methods. For the planar object, the APE indicates that Our-S outperforms other baseline methods. The visualization results of pose tracking are shown in the top row of Fig.~\ref{fig10}.

In regards to the cube, the results of the pose estimation are visually presented in Fig.~\ref{fig11}\subref{fig11.1}, while the quantified pose errors are shown in Table~\ref{tab2}. Our methods have the capability to continuously track the cube. Especially in the final stage, after the thirteenth second, the cube gradually moves out of the camera's field of view. As the pose errors of other methods start to escalate noticeably, our methods demonstrate greater robustness and stability in comparison. Among them, Our-MM achieves the highest accuracy. The trajectories are obtained by sequentially concatenating object poses, and illustrated after alignment with the ground truth to visualize the comparison results (Fig.~\ref{fig11}\subref{fig11.2}). As can be seen, the proposed methods yield lower trajectory errors compared to other methods. The second row of Fig.~\ref{fig10} demonstrates the partial pose tracking results. Even in scenarios where the object is partially occluded, Our-MM continues to exhibit exceptional performance. This is because we utilize the constant-velocity motion model to predict the object pose at the next moment, and effectively establish correspondences between visible events and lines.

\textbf{Complex object experiments.}
To further highlight the potential of our method, we also evaluate our methods on different challenging sequences of a complex planar object and a spacecraft model. During their movement, these objects exhibit a greater abundance and densification of lines, generate more events, and also introduce more noise. The distance thresholds remain unchanged, with ${{d}_{t}}=8$ pixels and $ d_m=0.5d_l $, as stated in the previous experiment. Due to an increased occurrence of disordered events impacting the precision of these methods, there has been an observed elevation in the overall APE depicted in Table~\ref{tab2}, when compared to the previous experiment.
Nonetheless, Our-MM has higher accuracy in solving poses of these complex objects. The pose tracking results of these two objects are shown in the last two rows of Fig.~\ref{fig10}. It is worth mentioning that, despite selecting fewer events in the event cluster compared to the previous experiment, our methods are still able to track these complex objects stably. 

In summary, we recommend a combination of the above methods to leverage their advantages fully: for simple planar objects, it is recommended to utilize Our-S. Compared to Our-M, which only uses the median as the weight value, Our-S takes into account the data distribution more effectively by utilizing the residual standard deviation. When dealing with complex objects, it is advisable to employ our Our-MM. During the movement of complex objects, an increased number of events are generated, simultaneously resulting in the introduction of a lot of noise. Our-MM demonstrates enhanced efficiency when addressing events that contain a large number of breakdown values.

\section{Conclusion}
\label{sec:concl}
We propose a line-based object pose estimation and tracking method for both planar and non-planar objects using an event camera. The method makes full use of the object model and event streams to provide an initial pose, establish event-line associations, and leverage a robust pose optimization and tracking method by minimizing event-line distances. Our experimental results demonstrate the effectiveness and robustness of the method both quantitatively and qualitatively. The proposed method is able to deal with event noise as well as outliers. It outperforms state-of-the-art methods, exhibiting higher robustness against noise and accurately solving the object pose even under partial occlusions. Future work will involve a generalization of the method to other cost functions, or its adaptation to handle more general scenarios, such as objects with curved edges rather than solely lines.

\bibliographystyle{IEEEtran}   
\bibliography{refe} 

\vspace{-15pt}
\begin{IEEEbiography}[{\includegraphics[width=1in,height=1.25in,clip,keepaspectratio]{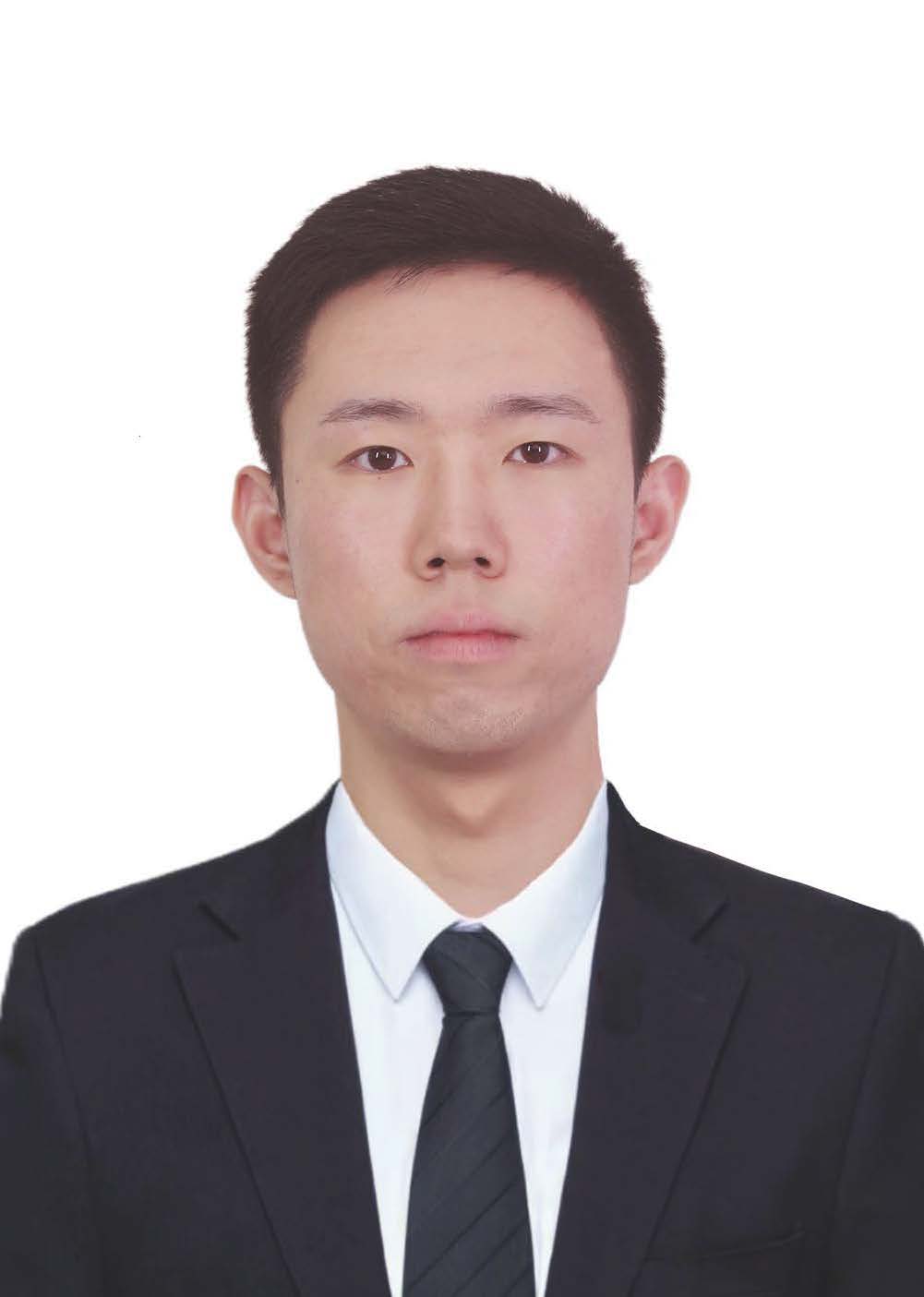}}]{Zibin Liu} received the B.E. degree from Harbin Engineering University, Harbin, China, in 2019 and the M.E. degree from the College of Aerospace Science and Engineering, National University of Defense Technology, Changsha, China, in 2021. He is currently pursuing the Ph.D. degree with the College of Aerospace Science and Engineering, National University of Defense Technology, Changsha, China. His research interests include photogrammetry and 3D vision.
\end{IEEEbiography}

\vspace{-10pt}
\begin{IEEEbiography}[{\includegraphics[width=1in,height=1.25in,clip,keepaspectratio]{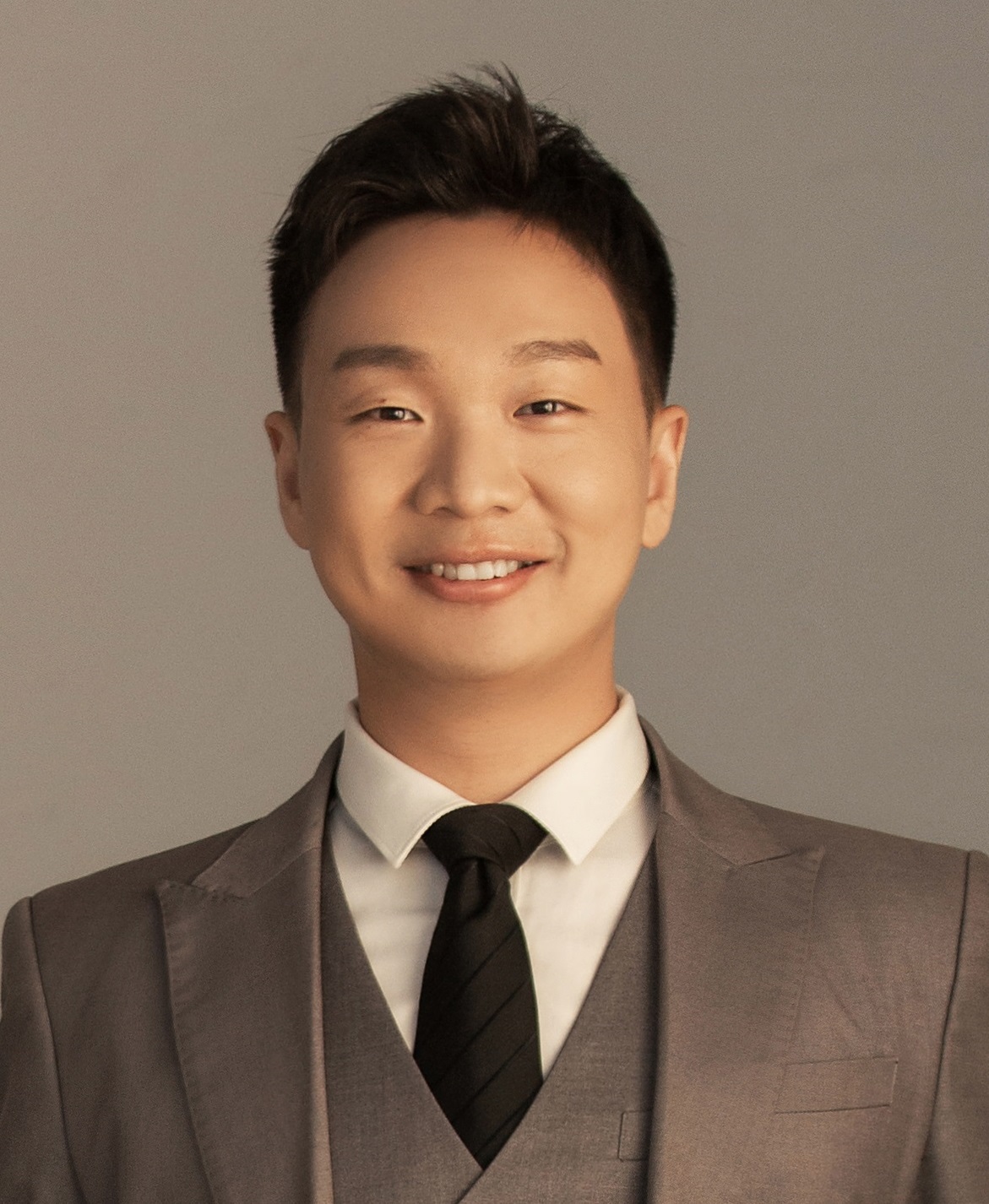}}]{Banglei Guan} (Member, IEEE) received the B.S. degree in Geomatics Engineering from Wuhan University in 2012, and the Ph.D. degree in Aeronautical and Astronautical Science and Technology from the National University of Defense Technology in 2018. From 2016 to 2017, he was a visiting PhD student at Graz University of Technology, advised by Prof. Horst Bischof and Prof. Friedrich Fraundorfer. He is currently an Associate Professor with the National University of Defense Technology. His research interests include Photomechanics and Videometrics. He has published research papers in top-tier journals and conferences, including IJCV, IEEE TIP, IEEE TCYB, CVPR, ECCV, and ICCV.
\end{IEEEbiography}

\vspace{-10pt}
\begin{IEEEbiography}[{\includegraphics[width=1in,height=1.25in,clip,keepaspectratio]{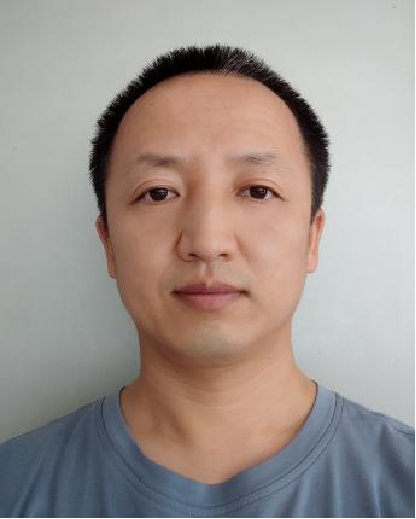}}]{Yang Shang} received the B.S. degree in aerodynamics, and the Ph.D. degree in aeronautical and astronautical science and technology from the National University of Defense Technology, Changsha, China, in 2000 and 2006, respectively. He is currently a Professor with the College of Aerospace Science and Engineering, National University of Defense Technology. His main research interests include photogrammetry, machine vision, and navigation systems.
\end{IEEEbiography}

\vspace{-10pt}
\begin{IEEEbiography}[{\includegraphics[width=1in,height=1.25in,clip,keepaspectratio]{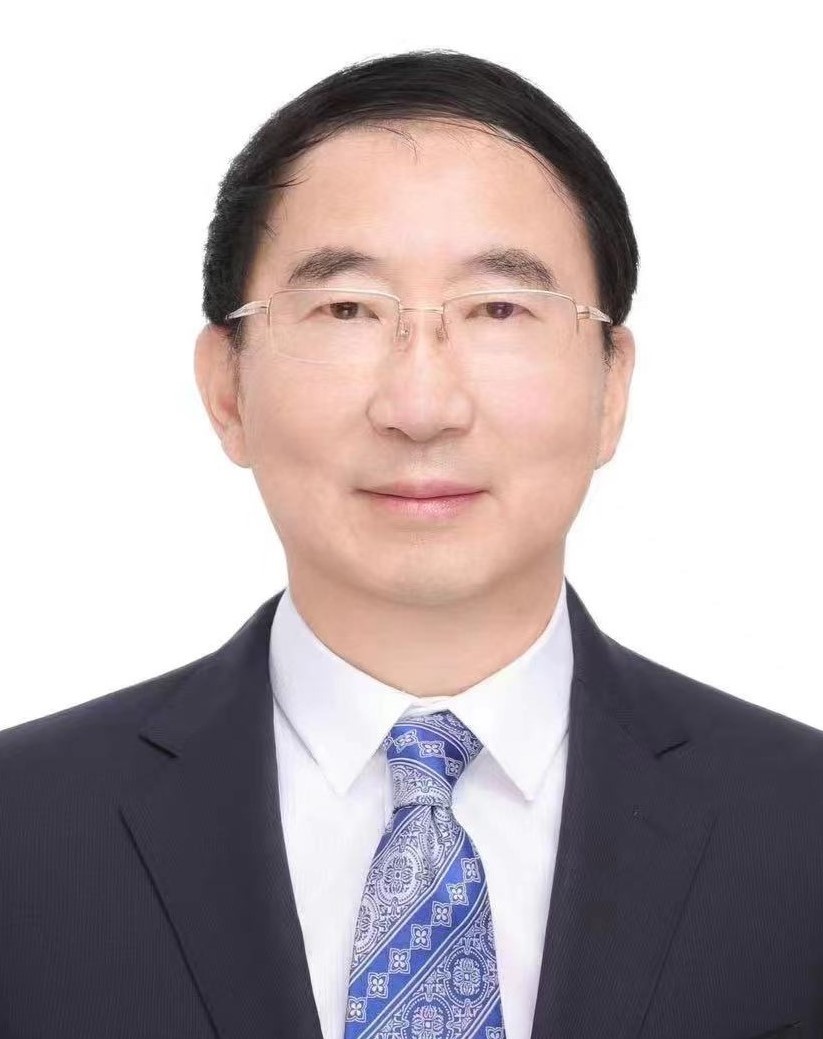}}]{Qifeng Yu} received the B.S. degree from Northwestern Polytechnic University, Xian, China, in 1981, the M.S. degree from the National University of Defense Technology, Changsha, China, in 1984, and the Ph.D. degree from Bremen University, Bremen, Germany, in 1996. Currently, he is a Professor at the National University of Defense Technology. He is a Member of the Chinese Academy of Sciences. He has authored three books and published over 200 papers. His current research fields are image measurement and vision navigation.
\end{IEEEbiography}

\vspace{-10pt}
\begin{IEEEbiography}[{\includegraphics[width=1in,height=1.25 in,clip,keepaspectratio]{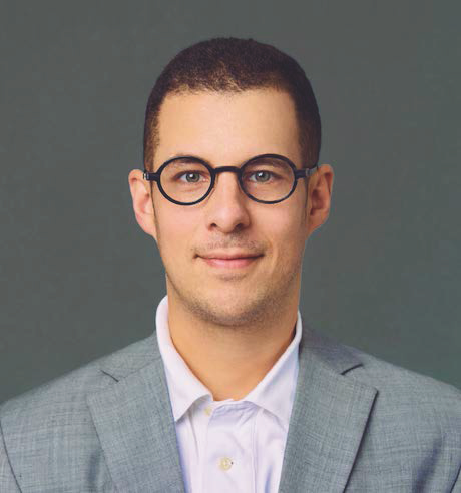}}]{Laurent Kneip} (Senior Member, IEEE) received the Dipl.-Ing. Univ. degree in mechatronical engineering from Friedrich-Alexander University, Erlangen, Germany, in 2008, and the Ph.D. degree from DMAVT, ETH Zurich, Zürich, Switzerland.

He was with Autonomous Systems Lab, Zürich. He is currently a Tenured Associate Professor with ShanghaiTech University, Shanghai, China, where he founded and directs the Mobile Perception Laboratory. He is also the Director of the ShanghaiTech Automation and Robotics Center, Shanghai. He has authored or coauthored in countless publications in top robotics and computer vision venues. He is the main author of OpenGV. His research focuses on enabling intelligent mobile systems to use vision for real-time 3-D perception of the environment.

Dr. Kneip was the recipient of the ARC Discovery Early Career Researcher Award (DECRA) in 2015, the Marr Prize (honourable mention) in 2017, International Young Scholar Award from NSFC in 2019, and the International Excellent Young Scientist Award from NSFC in 2022.
\end{IEEEbiography}

\vfill

\end{document}